\documentclass{article}

\RequirePackage{amsthm,amsmath,amsfonts,amssymb}
\usepackage{wasysym}
\usepackage{units}
\usepackage{url}
\RequirePackage[numbers]{natbib}
\RequirePackage[colorlinks,citecolor=blue,urlcolor=blue]{hyperref}
\RequirePackage{graphicx}
\usepackage[margin=1.5in]{geometry}

\makeatletter
\theoremstyle{plain}
\newtheorem{thm}{\protect\theoremname}

\newtheorem{prop}[thm]{\protect\propositionname}

\makeatother

\usepackage{babel}
\providecommand{\factname}{Fact}
\providecommand{\propositionname}{Proposition}
\providecommand{\theoremname}{Theorem}

\begin{document}

\title{Deep Learning as a Convex Paradigm of Computation: \\
Minimizing Circuit Size with ResNets}

\author{Arthur Jacot \\ Courant Institute, NYU \\ \texttt{arthur.jacot@nyu.edu}}

\maketitle


\begin{abstract}
This paper argues that DNNs implement a computational Occam's razor
- finding the `simplest' algorithm that fits the data - and that this
could explain their incredible and wide-ranging success over more
traditional statistical methods. We start with the discovery that
the set of real-valued function $f$ that can be $\epsilon$-approximated
with a binary circuit of size at most $c\epsilon^{-\gamma}$ becomes
convex in the `Harder than Monte Carlo'
(HTMC) regime, when $\gamma>2$, allowing for the definition of a
HTMC norm on functions. In parallel one can define a complexity measure
on the parameters of a ResNets (a weighted $\ell_1$ norm of the
parameters), which induce a `ResNet norm' on functions. The HTMC and
ResNet norms can then be related by an almost matching sandwich bound.
Thus minimizing this ResNet norm is equivalent to finding a circuit
that fits the data with an almost minimal number of nodes (within
a power of 2 of being optimal). ResNets thus appear as
an alternative model for computation of real functions, better adapted
to the HTMC regime and its convexity.
\end{abstract}

\section{Introduction}

Technically speaking, no statistical model can be strictly better
than another, and each model can be optimal under certain assumptions.
However, it has been argued that there is a statistical model that
outmatches all others up to constants: finding the minimal
Kolmogorov complexity function, i.e. the program with the minimal
description length, that fits the data \citep{Solomonoff_1964_part1,Solomonoff_1964_part2}.
Such an ideal statistical model would be at least as good as any model
with short descriptions \citep{Hutter_04_uaibook} which basically
includes all models that have been described within a single scientific
paper. But this perfect statistical model appears unattainable because
we know that Kolmogorov complexity is undecidable in general, and
even less optimizable.

In this paper, we prove results that suggest that Residual Networks
(ResNets) training is implementing a weaker version of this ideal
model: greedily minimizing circuit complexity in the so-called Harder
than Monte Carlo (HTMC) regime. These simplifications can be summarized
as follows:
\begin{itemize}
\item Minimal Circuit Size Problem (MCSP): searching for an interpolating
circuit with minimal minimal number of operations is ``only'' NP.
Minimal circuit size interpolators share a lot of the statistical
power of minimal Kolmogorov complexity ones, in particular their ability
to learn hierarchical structures thanks to its compositionality. Roughly
speaking, the tasks where minimal circuit size interpolators are suboptimal
are those where the runtime of the interpolator is significantly larger
than its description length, or in other terms its effective number
of parameters.
\item Real-valued function computation: DNNs are real-valued, so it makes
sense to focus on the circuit complexity $C(f,\epsilon)$ of real-valued
functions which is the smallest circuit size required to approximate
$f$ within an $\epsilon$ error. A central quantity is the rate $\gamma>0$
of approximation (i.e. $C(f,\epsilon)=O(\epsilon^{-\gamma})$), which
also determines the rate of the generalization error, e.g. for the
MSE on $N$ datapoints, one has $R_{test}=O(N^{-\frac{2}{2+\gamma}})$.
\item Harder than Monte Carlo (HTMC) regime ($\gamma>2$): A surprising
discovery is that the set of functions that are approximable with
a rate $\gamma$ becomes convex when $\gamma>2$, which we call the
HTMC regime, thus allowing us to define a HTMC norm $\left\Vert f\right\Vert _{M^{\gamma}}$.
This paper focuses on the minimization of this norm. Empirical computations
of the test error rate of Large Language Models (LLMs) have identified
rates such as $N^{-0.095}$ \citep{kaplan2020scaling} which corresponds
to a rate $\gamma\approx18$ that is far into the HTMC regime.
\end{itemize}

\subsection{Related Works}

\textbf{Universal Prior:} The idea of searching for the interpolator
with minimal Kolmogorov complexity is actually 5 years older \citep{Solomonoff_1964_part1,solomonoff_1960_preliminary,Solomonoff_1964_part2}
than Kolmogorov complexity itself \citep{kolmogorov_1965_three_approaches_kolg_complexity}.
However, the uncomputability of this task has led this idea to be
considered more as a thought experiment than a concrete model \citep{donoho_02_Kolmogorov_sampler,Hutter_04_uaibook},
although some computable relaxations have been proposed \citep{veness_2011_MC_AIXI},
which can be interpreted as optimizing for a notion of complexity
that is somewhat closer to circuit complexity.

\textbf{Minimum Circuit Size Problem (MCSP):} The problem of minimizing
circuit size has been known to be NP for a long time, but it remains
unknown whether it is NP complete or not \citep{hitchcock_2015_NP_complete_MCSP,Ilango_2020_multidim_MCSP_NP_hard}.
The interest in the MCSP started in the 1950s in Russia as a typical
example of a problem that requires brute force search \citep{Trakhtenbrot_1984_history_perebor},
and later Leonid Levin actually wanted to prove NP-completeness of the
MCSP in his seminal work on P=NP \citep{Levin_webpage_hardness_of_search_MCSP}
but ultimately failed \citep{Levin_1973_P_NP}. To this day, ``no
known algorithm significantly improves over brute-force search''
\citep{Ilango_2023_MCSP_SAT}. The idea that there is a setting where
MCSP becomes a convex problem, and that neural networks can then be
used to (even greedily or approximately) solve it, is a stark departure
from previous approaches.

\textbf{Relating DNNs to Circuits:} The analogy between DNNs and circuits
is quite self-evident and has guided a lot of empirical research \citep{elhage2022superposition}.
More formally, the space of functions representable by DNNs with a
bounded width or number of non-zero parameters can be related to computational
spaces, e.g. the space of compositionally sparse functions \citep{poggio_2017_curse_of_dim_compositionality,Bauer_2019_compositional_learning_bounded_width,Schmidt-Hieber_2020_compositional_sparsity,poggio2024compositional,danhofer2025position}
which resemble computational graphs of Hölder functions where each
function has low input dimension. The ability of these spaces to represent essentially any algorithm have been used to guarantee
strong generalization bounds in other settings too \citep{chen_2023_diffusion_subspace_learning,oko2023diffusion}.
However, the number of non-zero parameters remains a discrete quantity
that is not clearly optimizable, so this line of work has generally
failed to capture the training dynamics.

\textbf{Parameter norm minimization:} Another source of inspiration
is the literature DNNs whose parameters have bounded $L_2$ norm, for example as a result of $L_2$ regularization. This type of analysis has been very successful
in simple models such as linear networks \citep{dai_2021_repres_cost_DLN,woodworth2019kernel},
or in shallow networks \citep{Barron_1993_UniversalAB,bach2017_F1_norm,weinan_2019_barron,Savarese_2019_repres_bounded_norm_shallow_ReLU_net_1D,Ongie_2020_repres_bounded_norm_shallow_ReLU_net},
where norms on functions have been identified that describe exactly
the space of functions that can be represented with a bounded $L_2$ norm of the parameters. In these simpler networks, the convexity of those norms in function space translates
into an absence of barriers between local and global minima in function
spaces \citep{achour2021,jacot-2021-DLN-Saddle}, and more generally
played an important role in the development of convergence proofs
\citep{Chizat2018}.

For deep fully-connected networks it is instead a notion of rank (on
non-linear functions) which captures the space of representable functions
\citep{jacot_2022_BN_rank,jacot_2023_bottleneck2}, at least at first
order. Alternatively, by switching to products of Froebenius norms of the weights matrices, one obtains a "Neural Hilbert ladder" \cite{chen_2024_neural_hilbert_ladder} and resulting in function spaces that are convex up to a constant that blows up exponentially in depth (see Section A.3 in \cite{chen_2024_neural_hilbert_ladder}). This suggests that taking
large depth limits induces compositionality at the cost of losing convexity,
as illustrated by the resulting norm vs rank notions, which did not
bode well for the possibility of proving convergence results. Thankfully,
this present paper shows that (approximate) sub-additivity under both addition and composition is actually possible, though it requires switching to another notion of parameter norm.

\textbf{Generalization Bounds:} Looking for other notions of parameter
complexity that avoid the pathologies of the $L_{2}$-norm, I have
taken inspiration from the diverse literature on generalization bounds
for DNNs, based on different notions of parameter which naturally
avoid the issues of the $L_{2}$-norm.

A first line of work, generally based on Rademacher complexity, is
based on taking products rather than sums of Frobenius norms over
the layers \citep{chen_2024_neural_hilbert_ladder,neyshabur2015_path_norm}
or the related path norms \citep{barron2019_path_norm_complexity,weinan_2019_barron}.
These can exhibit different degrees of approximate convexity, but
they are generally sub-multiplicative rather than sub-additive under
composition, thus leading to loose bounds on ``deep'' computational
graphs.

Another line of work relies on covering number arguments to achieve
some form of compositional sub-additivity \citep{bartlett_2017_composition_generalization,wei_2019_generalization_without_Lipschitz,hsu_2021_generalization_distillation},
leading to almost tight rates on compositions of Sobolev functions
\citep{jacot2024_covering_generalization}. A key idea is that the
contribution of a connection to the parameter norm should be weighted
by how much this connection affects the outputs, which is one of the
key ideas that help obtain a much stronger approximate convexity in
this paper.

Still, in practice these generalization bounds tend to be very loose
\citep{jiang_2019_fantastic_generalization} and it is unclear what
makes them suboptimal. This paper started from the realization that
these covering number arguments can be translated into circuit size
upper bounds, and then providing an almost matching lower bound which
could be key in understanding how to improve these types of bounds
and compare them.

\section{Setup and Main Results}

The main result of this paper is a sandwich bound that relates the
so-called HTMC norm $\left\Vert f\right\Vert _{M^{\gamma}}$ and ResNet
(pseudo-)norm $\left\Vert f\right\Vert _{R^{\omega}}$, which measure
the computability of $f$ in with respect to two radically different
notions of computation: the first being the traditional notion of
``real-to-bin'' computation where a binary circuit operates on a
binary representation of the inputs, while the second proposes a new
notion of computability based on being approximable by a ResNet $f_{\theta}$
with a small parameter complexity $R(\theta)$. The first is a discrete
notion of computation while the second is a continuous one and our
main result shows that these two paradigm can be almost
unified, which implies that DNN optimization is almost equivalent to the MCSP.

\subsection{Real-to-bin computation} We consider binary circuits $C$
made up of binary $\text{AND},\text{OR}$ and $\text{NOT}$ nodes.
We write $\left|C\right|$ for the number of nodes, and we write
\[
\left|C\right|_{p}=\min_{n_{1},\dots,n_{L}}\left(\sum_{\ell=1}^{L}n_{\ell}^{p}\right)^{\frac{1}{p}}
\]
where the minimum is over all ways in which to organize the nodes
of the circuit into $L$ layers, each with $n_{1},\dots,n_{L}$ nodes
(so that $\left|C\right|=\sum n_{\ell}=\left|C\right|_{p=1}$). We
will be only interested in the case $p=\frac{2}{3}$, which can be
interpreted as favoring parallel circuit that can be organized in
only a few layers with many nodes per layer. We have the bounds $\left|C\right|\leq\left|C\right|_{\frac{2}{3}}\leq L^{\frac{1}{3}}\left|C\right|$,
for $L$ the optimal depth (since $(\sum n_{\ell})^{\frac{2}{3}}\leq\sum n_{\ell}^{\frac{2}{3}}$
and by concavity $\frac{1}{L}\sum n_{\ell}^{\frac{2}{3}}\leq\left(\frac{1}{L}\sum n_{\ell}\right)^{\frac{2}{3}}$).

A real-valued function $f:\mathbb{R}^{d_{in}}\to\mathbb{R}^{d_{out}}$
cannot be computed exactly, but it can be approximated by a circuit
which takes the binary representation of the inputs $x$ and returns
a binary representation of the outputs. Given a norm $\left\Vert \cdot\right\Vert $
on functions supported over a bounded hyper-rectangle (e.g. the $L_{2}(\pi)$
norm $\left\Vert \cdot\right\Vert _{\pi}$ over a distribution $\pi$
or the $L_{\infty}$ norm $\left\Vert \cdot\right\Vert _{\infty}$),
we define 
\[
C_{p}(f,\epsilon;\left\Vert \cdot\right\Vert )=\min_{C:\left\Vert C-f\right\Vert \leq\epsilon}\left|C\right|_{p}
\]
where the minimum is over all circuits that $\epsilon$-approximate
$f$. We will mainly focus on $C_{p=1}(f,\epsilon; L_2(\pi))$ which we will write as $C(f,\epsilon)$, but other cases will appear as well, such as $C_{p=\frac{2}{3}}(f,\epsilon;L_{\infty})$.

We can classify computable real-valued functions in terms of their
rate $\gamma$ and prefactor $c$ of computability: $C_{p}(f,\epsilon;\left\Vert \cdot\right\Vert )\leq c\epsilon^{-\gamma}$.
For any $\gamma$, we define 
\[
\left\Vert f\right\Vert _{M_{p}^{\gamma}(\left\Vert \cdot\right\Vert )}^{\gamma}=\max_{\epsilon}\epsilon^{\gamma}C_{p}(f,\epsilon;\left\Vert \cdot\right\Vert )
\]
which is the smallest value such that $C_{p}(f,\epsilon;\left\Vert \cdot\right\Vert )\leq\left\Vert f\right\Vert _{M_{p}^{\gamma}(\left\Vert \cdot\right\Vert )}^{\gamma}\epsilon^{-\gamma}$.
For simplicity, we will denote $M^{\gamma}$ for $M_{p=1}^{\gamma}(L_2(\pi))$. 

The homogeneity of $\left\Vert f\right\Vert _{M_{p}^{\gamma}(\left\Vert \cdot\right\Vert )}$
follows directly from its definition, however it is not subadditive
in general, thus failing to be a norm. Surprisingly, when $\gamma>2$,
which we call the the Harder than Monte Carlo (HTMC) regime, $\left\Vert f\right\Vert _{M^{\gamma}}$
becomes convex up to a constant:
\begin{thm}
For $\gamma>2$, there is a constant $c_{\gamma}$ such that for all
$m\geq1$, $\left\Vert \sum_{i=1}^{m}f_{i}\right\Vert _{M^{\gamma}}\leq c_{\gamma}\sum_{i=1}^{m}\left\Vert f_{i}\right\Vert _{M^{\gamma}}$.
\end{thm}
\begin{proof}[Sketch of proof] This follows quite directly from Multi-level
Monte Carlo (MLMC) \citep{giles2015_MLMC}. See Theorem \ref{thm:convexity_HTMC} in Section \ref{subsec:Basic-properties-HTMC} for the complete proof.
\end{proof}

This does not quite make $M^{\gamma}$ a norm, but there is a norm
$Conv(M^{\gamma})$ (obtained by replacing the unit ball of $M^{\gamma}$
by its convex hull) such that $c_{\gamma}^{-1}\left\Vert f\right\Vert _{M^{\gamma}}\leq\left\Vert f\right\Vert _{Conv(M^{\gamma})}\leq\left\Vert f\right\Vert _{M^{\gamma}}$,
thanks to the fact that the constant $c_{\gamma}$ does not depend
on $m$.

The convexity of the HTMC norm suggests that one could use convex optimization
methods to minimize circuit size, but it remains unclear how to
optimize over the infinite sequence of circuits $A_{1},A_{2},\dots$ that
underlies the definition of the HTMC norm. The issue is that the
usual notion of real-to-bin computation is fundamentally discrete,
and thus ill adapted to leverage the convex/continuous structure of
computation in the HTMC regime. This could be solved by switching
to a ``fully real'' notion of computation: computational graphs
with nodes that take real values and apply continuous operations rather
than binary ones, and with connections between nodes that take any
value instead of being $0$ or $1$ only. This pretty much describes
a neural network, and so we turn our attention to ResNets and how
they can be connected to the HTMC norm.

\subsection{ResNet computation} For an input $x\in\mathbb{R}^{d_{in}}$, we define the activations $\alpha_{0}(x),\dots,\alpha_{L}(x)\in\mathbb{R}^{d+1}$
of a Residual Network (ResNet) of depth $L$ by first fixing their last coordinate to $\alpha_{\ell,d+1}(x)=1$ (this will allow us to combine the usual weight matrix and bias vector into a single weight matrix) and then defining the first $d$ coordinates $\alpha_{\ell,1:d}(x)$ recursively:
\begin{align*}
\alpha_{0,1:d}(x) & =W_{in}x\\
\alpha_{\ell,1:d}(x) & =\alpha_{\ell-1,1:d}(x)+W_{\ell}\sigma\left(V_{\ell}\alpha_{\ell-1}(x)\right),
\end{align*}
where $\sigma(x)=\max{x,0}$ is the ReLU, and the matrices $W_\ell,V_\ell$ are of dimensions $d\times w$ and $w \times d+1$ respectively. We call $w$ the width of the network and $d$ the hidden dimension. Finally, the outputs are defined as $f_{\theta}(x)=W_{out}\alpha_{L}(x)$,
where $\theta$ is the vector of parameters, obtained by concatenating
the entries of $W_{in},W_{out}$ and all $W_{\ell},V_{\ell}$s.
We will focus on inputs contained within a hyper-rectangle with sides
$s_{1},\dots,s_{d_{in}}$ which we combine into a $d_{in}\times d_{in}$
diagonal matrix $S$.

Our strategy is to define a notion of parameter complexity $R(\theta)$
that describes how ``prunable'' the network $f_{\theta}$ in the sense
that one can remove all but $O(R(\theta)^{2}\epsilon^{-2})$ connections
of the network without changing the outputs by more than an $\epsilon$.
This in turns implies that the function $f_{\theta}$ can be approximated
by a circuit of size $O(R(\theta)^{2}\epsilon^{-2})$ (up to $\log\epsilon^{-1}$
terms) and therefore implies a bound of the type ``$\left\Vert f_{\theta}\right\Vert _{M^{\gamma=2}}\apprle R(\theta)$''.

Before we define $R(\theta)$, we give a simple heuristic argument
that leads to another $R_{lin}(\theta)$ that is both simpler and more
easily interpretable. Sadly, to obtain rigorous results we need to
use the more complex $R(\theta)$. The intuition goes as follows:
if we prune each parameter $\theta_{i}$ randomly, by multiplying
it with independent rescaled Poisson random variables $\tilde{\theta}_{i}=\theta_{i}\frac{P_{i}}{\lambda_{i}}$
where $P_{i}\sim Poisson(\lambda_{i})$, then by a simple Taylor approximation,
we obtain
\[
\mathbb{E}\left\Vert f_{\theta}-f_{\tilde{\theta}}\right\Vert _{\pi}^{2}\approx\sum\lambda_{i}^{-1}\theta_{i}^{2}\left\Vert \partial_{\theta_{i}}f_{\theta}\right\Vert _{\pi}^{2}
\]
with an expected number of non-zero parameters of $\mathbb{E}\left\Vert \tilde{\theta}\right\Vert _{0}=\sum1-e^{-\lambda_{i}}\leq\sum\lambda_{i}$.
With the optimal choice $\lambda_{i}$s, one obtains an error
$\epsilon$ with $\mathbb{E}\left\Vert \theta\right\Vert _{0}\leq R_{lin}(\theta)^{2}\epsilon^{-2}$
for 
\[R_{lin}(\theta)=\sum\left|\theta_{i}\right|\left\Vert \partial_{\theta_{i}}f_{\theta}\right\Vert _{\pi}. \]

But to control the other terms in the Taylor expansion, we will use
the following complexity measure:
\begin{align*}
R(\theta)^{\frac{2}{3}}= & \min_{D_{\ell}}\sum_{\ell=1}^{L}\left\Vert D_{\ell}\left|W_{\ell}\right|\left|V_{\ell}\right|C_{\ell-1}\right\Vert _{1}^{\frac{2}{3}}+\left\Vert D_{0}W_{in}S\right\Vert _{1}^{\frac{2}{3}}+\left\Vert W_{out}C_{L}\right\Vert _{1}^{\frac{2}{3}}
\end{align*}
where $C_{\ell}$ is diagonal with $C_{\ell,ii}=\left\Vert \alpha_{\ell,i}\right\Vert _{\infty}$,
and the minimum is over all diagonal $D_{\ell}$s such that $Lip((\alpha_{\ell}\to f_{\theta})\circ D_{\ell}^{-1})\leq1$
(i.e. the map from the $\ell$-th layer activations $\alpha_{\ell}$
to the outputs precomposed with $D_{\ell}^{-1}$ is $1$-Lipschitz),
and we apply the absolute value entrywise to the matrices $W_{\ell},V_{\ell}$.

One can verify that $2R(\theta)$ is an upper bound on the `ideal' parameter
complexity $R_{lin}(\theta)$. Note that both $R(\theta)$ and $R_{lin}(\theta)$
are weighted $\ell_{1}$ norms of the parameters, but since the weights
depend on the parameter themselves, $R(\theta)$ and $R'(\theta)$
are generally not convex in parameter space, hence why we call them
`complexities' and do not use a norm notation.

Convexity does reveal itself in the function space, mirroring the
HTMC convexity. We define the $\omega$-ResNet norm $\left\Vert f\right\Vert _{R^{\omega}}$
that captures functions $f$ that can be $\epsilon$-approximated
by a ResNet with a complexity measure $R(\theta)$ of order $\epsilon^{1-\omega}$:
\[
\left\Vert f\right\Vert _{R^{\omega}(\left\Vert \cdot\right\Vert )}^{\omega}=\max_{\epsilon}\epsilon^{\omega-1}\min_{\theta:\left\Vert f_{\theta}-f\right\Vert \leq\epsilon}R(\theta),
\]
which implies that for all $\epsilon$, there are parameters $\theta_{\epsilon}$
such that $\left\Vert f_{\theta_{\epsilon}}-f\right\Vert \leq\epsilon$
and $R(\theta_{\epsilon})\leq\left\Vert f\right\Vert _{R^{\omega}(\left\Vert \cdot\right\Vert )}^{\omega}\epsilon^{1-\omega}$.
The choice $\omega=1$ corresponds to functions that can be represented
exactly with a finite complexity measure $R(\theta)$, and we will sometimes use the notation $\|\cdot\|_R = \|\cdot\|_{R^{\omega=1}}$. One can easily check that $\left\Vert f\right\Vert _{R_{lin}^{\omega}}$ is convex because it is subadditive under putting two networks in parallel,
while $\left\Vert f\right\Vert _{R^{\omega}}$ only satisfies
\[\left\Vert f+g\right\Vert _{R^{\omega}}\leq\left(\left\Vert f\right\Vert _{R^\omega}^{\frac{2\omega}{1+2\omega}}+\left\Vert g\right\Vert _{R^\omega}^{\frac{2\omega}{1+2\omega}}\right)^{\frac{1+2\omega}{2\omega}}.\]

One of the goal of this paper is to prove that the circuit size required to approximate a ResNet can be bounded in terms of the parameter complexity measure $R(\theta)$ independently from the size of the network. For this reason, we will generally assume that the depth $L$, width $w$, and hidden dimension $d$ are arbitrarily large, and obtain results that are independent of these quantities (as long as they are large enough). For example, in the definition of the $R^\omega$ norm, the minimization is over all possible network sizes and parameters.

\subsection{Sandwich bound} We are now able to state our main result:
\begin{thm}\label{thm:main_sandwich}
For all $\omega>1$ and $\delta>0$, one has (up to constants that
are polynomial in the input and output dimensions):
\[
\left\Vert f\right\Vert _{M^{\gamma=2\omega+\delta}(L_2(\pi))}+\left\Vert f\right\Vert _{C^{\alpha=\frac{1}{\omega}}}\lesssim\left\Vert f\right\Vert _{R^{\omega}(L_\infty)}\lesssim\left\Vert f\right\Vert _{M_{p=\frac{2}{3}}^{\gamma=\omega-\delta}(L_{\infty})}+\left\Vert f\right\Vert _{C^{\alpha=\frac{1}{\omega-\delta}}},
\]
where $\left\Vert \cdot\right\Vert _{C^{\alpha}}$ is the Hölder (semi-)norm
for $\alpha\in(0,1]$, and where $\lesssim$ denotes an inequality up to a constant that depends $d_{in},d_{out},\gamma,\alpha,\delta$.
\end{thm}

\begin{proof}[Structure of the proof] The lower bound is a pruning bound, where the majority of the weights
of the ResNet are removed to obtain a small circuit, and the upper
bound is a construction bound, where an infinite sequence of circuits
that approximate a function $f$ are combined into a single ResNet. The LHS follows from Theorem \ref{thm:LHS_sandwich} from Section \ref{sec:pruning_bound} and point 2 of Proposition \ref{prop:properties_ResNet_norm} from Section \ref{subsec:Basic-properties-ResNet}. The RHS follows from Theorem \ref{thm:RHS_sandwich} in  Section \ref{subsec:construction_bound}.
\end{proof}

The convexity of the HTMC regime hints at the potential to use of convex
optimization to solve the MCSP, but it is unclear how one could differentiate
the HTMC norm because it is defined in terms of a discrete structure:
infinite sequences of binary circuits. In contrast the linearized
complexity measure $R_{lin}(\theta)$ is both differentiable and convex
in function space, which could be leveraged to prove global convergence
of gradient descent (in the infinite width limit). The $R(\theta)$
norm is less practically differentiable because of its reliance on
Lipschitz constants and $L_{\infty}$-norms off the activations, it
is also not exactly convex (though it could easily be ``convexified''
by taking ensembles of ResNets). A future goal is to find a differentiable, convex
complexity measure that admits a similar sandwich bound, and use it to prove global convergence of gradient descent to an approximate
solution of the MCSP, opening the door to a whole new family of convex
algorithms for this fundamental problem.

The quality of this approximate solution directly depends on the tightness of the sandwich bound. The $\delta$ reflects the presence of logarithmic terms,
some of which could be removed, and the $p=\frac{2}{3}$ and $L_{\infty}$
norm on the RHS are likely artifacts of the current proof and/or of
the definition of $R(\theta)$. There are however two distinctions that might be representative of something more fundamental:

\emph{Hölder continuity:} The presence of the Hölder-continuity norm
on both sides implies that the $R$-norm captures a mix of computability
and regularity. This is interesting, because especially in high-dimension,
these two notions are essentially ``orthogonal'': there are computable
functions, such as the heavy-side, that are $\gamma$-computable for
all $\gamma>0$ but not Hölder, and while all Hölder functions are
computable, the worst-case and typical rate suffers from the curse
of dimensionality because the number of equal frequency Fourier modes is exponential
in $d_{in}$ (computing only the lowest Fourier frequencies,
one obtains $\left\Vert f\right\Vert _{M^{\gamma=\frac{d_{in}}{\alpha}}}\lesssim\left\Vert f\right\Vert _{C^{\alpha}}$). The HTMC unit ball is also much ``smaller'' than the Hölder unit ball in high dimension, as illustrated by the fact that uniform generalization bound over these balls are of order $N^{-\frac{2}{2+\gamma}}$ and $N^{-\frac{2\alpha}{2\alpha+d_{in}}}$ respectively for the Mean squared Error (MSE), where only the second one suffers from the curse of dimensionality.

\emph{HTMC rate gap:} Up to $\delta$ terms, the HTMC rate of the
LFS is $\gamma=2\omega$ while the RHS is $\gamma=\omega$, which
implies that if one translates a circuit into a ResNet and back into
a circuit, the number nodes can be squared in the worst case. This
also implies that if one minimizes $R(\theta)$ as a proxy for the
circuit size, one can only guarantee recovery of a circuit whose size
is at most the square of the optimal one. A possible explanation for
this gap is that while the Hölder continuity guarantees regularity
of the map from input to output, the complexity measure $R(\theta)$
also controls the regularity of the intermediate steps in the middle
of the network. It is possible that this worst-case gap only applies
to functions whose optimal circuit approximation has intermediate
representations that are in some sense highly irregular (e.g. the
approximating circuits for different $\epsilon$ are completely different,
so that the intermediate binary values cannot be interpreted as converging
to an intermediate real-valued representation as $\epsilon\searrow0$). In fact if we restrict ourselves
to compositions $f=g_{L}\circ\dots\circ g_{1}$ of Sobolev and Lipschitz
functions $g_{1},\dots,g_{L}$ (or more general computational graphs
of Sobolev functions, similar to notion the compositional sparsity
\citep{poggio_2017_curse_of_dim_compositionality}), then we can guarantee
a rate of $\gamma=\max_{\ell}\frac{d_{\ell}+3}{k_{\ell}}$, where
$k_{\ell}$ is the differentiability of $g_{\ell}$ and $d_{\ell}$
its input dimension, which is much closer to the optimal $\gamma^{*}=\max_{\ell}\frac{d_{\ell}}{k_{\ell}}$
than $2\gamma^{*}$ (this follows from Theorem 5 in \citep{jacot2024_covering_generalization},
together with the fact that the complexity of this paper upper bounds
the one of this paper up to constants in the intermediate dimensions).

\section{Harder than Monte Carlo (HTMC) Regime} \label{sec:HTMC} 
This section delves deeper into the HTMC regime and related HTMC norm. We first prove a few basic properties in Section \ref{subsec:Basic-properties-HTMC}, in particular the convexity of the HTMC norm. We then prove a probably approximately correct (PAC) generalization bound based on the HTMC norm in Section \ref{subsec:Generalization_bounds}. Finally in Section \ref{subsec:Tetrakis_functions} we introduce a family of functions, the Tetrakis functions, which we interpret as being an approximation for the vertices or extrema points of the HTMC unit ball. These basic properties and Tetrakis functions which will play a central role in Section \ref{subsec:ResNets} where we will prove Theorem \ref{thm:main_sandwich}.

\subsection{Basic Properties}\label{subsec:Basic-properties-HTMC}

It follows directly from the definition of the HTMC norm that it is
homogeneous and bounds the underlying norm $\left\Vert \cdot\right\Vert $:
\begin{prop}
We have:
\begin{enumerate}
\item $\left\Vert \lambda f\right\Vert _{M_{p}^{\gamma}(\left\Vert \cdot\right\Vert )}=\lambda\left\Vert f\right\Vert _{M_{p}^{\gamma}(\left\Vert \cdot\right\Vert )}$,
\item $\left\Vert f\right\Vert \leq\left\Vert f\right\Vert _{M_{p}^{\gamma}(\left\Vert \cdot\right\Vert )}$.
\end{enumerate}
\end{prop}
\begin{proof}
(1) If the algorithm $A$ approximates $f$ within an $\epsilon$, then $\lambda A$ approximates $\lambda f$ within an $\lambda \epsilon$ error and therefore
\[ 
C_p(\lambda f, \lambda \epsilon,\|\cdot\|) = C_p(f,\epsilon,\|\cdot\|).
\]
In our setting, $A$ and $\lambda A$ have the same size because we are only changing the decoding step, i.e. how the final binary representation gets mapped to a real value, which is not counted in the size.

We then have
\begin{align*}
\|\lambda f\|^\gamma_{M^\gamma_p(\|\cdot\|)} &= \min_\epsilon \epsilon^\gamma  C_p(\lambda f, \epsilon;\|\cdot\|)\\ &= \min_\epsilon (\lambda \epsilon)^\gamma  C_p(\lambda f, \lambda \epsilon;\|\cdot\|) \\ &= \lambda^\gamma  \min_\epsilon C_p(f, \epsilon;\|\cdot\|)  \\ &=\lambda^\gamma \|f\|^\gamma_{M^\gamma_p(\|\cdot\|)}.
\end{align*}

(2) For any $\epsilon>\|f\|_{M^\gamma_p(\|\cdot\|)}$, there is a circuit that $\epsilon$-approximates $f$ with a circuit size of zero, because $C(f,\epsilon)\leq\|f\|^\gamma_{M^\gamma_p(\|\cdot\|)} \epsilon^{-\gamma}<1$. Since the only circuit with zero nodes is the constant $0$ function, we obtain that $ \|f\| = \|f-0\|\leq \epsilon$. Since $\|f\| \leq \epsilon$ for all $\epsilon > \|f\|_{M^\gamma_p(\|\cdot\|)}$, we have $\|f\| \leq \|f\|_{M^\gamma_p(\|\cdot\|)}$.
\end{proof}

However, subadditivity fails in general, because if we approximate
$f+g$ by the sum of the approximations of $f$ and $g$, we obtain
(disregarding the cost of the final summation which is of order $\log\epsilon$)
for all $p\geq1$
\[
C_{p}(f+g,\epsilon;\left\Vert \cdot\right\Vert )\leq\min_{\epsilon=\epsilon_{1}+\epsilon_{2}}C_{p}(f,\epsilon_{1};\left\Vert \cdot\right\Vert )+C_{p}(g,\epsilon_{2};\left\Vert \cdot\right\Vert ).
\]
Optimizing over $\epsilon_{1}$ and $\epsilon_{2}$, we only obtain
a weaker notion of subadditivity: $\left\Vert f+g\right\Vert _{M_{p}^{\gamma}(\left\Vert \cdot\right\Vert )}^{\frac{\gamma}{\gamma+1}}\leq\left\Vert f\right\Vert _{M_{p}^{\gamma}(\left\Vert \cdot\right\Vert )}^{\frac{\gamma}{\gamma+1}}+\left\Vert g\right\Vert _{M_{p}^{\gamma}(\left\Vert \cdot\right\Vert )}^{\frac{\gamma}{\gamma+1}}$,
and the case $p<1$ is even worse.

But we can do better than this naive bound, by leveraging the fact
that we have access to multiple approximations of $f$ across any
accuracy $\epsilon$, and by using a Multi Level Monte Carlo (MLMC)
we can prove the convexity of $M^{\gamma}$ up to a universal constant:
\begin{thm}
\label{thm:convexity_HTMC}For all $\gamma>2$, there is a constant
$c_{\gamma}$ such that for all $m\geq1$, \[\left\Vert \sum_{i=1}^{m}f\right\Vert _{M^{\gamma}}\leq c_{\gamma}\sum_{i=1}^{m}\left\Vert f\right\Vert _{M^{\gamma}}.\]

\end{thm}

\begin{proof} The proof relies on the probabilistic method, we define a random approximator and bound its expected squared error, which implies the existence of an approximator.

Thanks to the homogeneity of the HTMC norm, we may assume that $\sum_{j=1}^m \| f_j \|_{M^\gamma }=1$ and write $p_j=\| f_j \|_{M^\gamma }$. For an integer $k$ we sample $n_{k}$ iid random functions $f_{k,1},\dots,f_{k,n_k}$, each equal to $p_j^{-1}f_j$ with probability $p_j$ so that $\mathbb{E}[f_{k,i}]=\sum_{j=1}^m f_j$ and $\|f_{k,i}\|_{M^\gamma}=1$ almost surely. We then consider the pairs of circuits $A_{k,i},B_{k,i}$
that are $2^{-k}$ and $2^{-k+1}$ approximators of $f_{k,i}$, so that $\left\Vert A_{k,i}-B_{k,i}\right\Vert _{\pi}\leq\left\Vert A_{k,i}-f_{k,i}\right\Vert _{\pi}+\left\Vert B_{k,i}-f_{k,i}\right\Vert _{\pi}\leq3\cdot2^{-k}$
and $A_{k,i}$, $B_{k,i}$ have circuit size at most $2^{\gamma k}$
and $2^{\gamma(k-1)}$ respectively. We then approximate
$f=\sum_{j=1}^m f_j$ using MLMC
\[
\tilde{f}=\sum_{k=0}^{k_{max}}\frac{1}{n_{k}}\sum_{i=1}^{n_{k}}A_{k,i}-B_{k,i}.
\]
Note that we may assume that $B_{0,i}=0$ since it approximates $f_{0,i}$ within $2^{-0+1}=2$ because $\|f\|_\pi\leq\|f\|_{M^\gamma}= 1$. We have
\begin{align*}
\left\Vert \mathbb{E}\tilde{f}-f\right\Vert _{\pi} & =\left\Vert \mathbb{E}A_{k_{max},i}-f\right\Vert _{\pi}\leq2^{-k_{max}}\\
\mathbb{E}\left\Vert \tilde{f}-\mathbb{E}\tilde{f}\right\Vert _{\pi}^{2} & \leq\sum_{k=0}^{k_{max}}\frac{1}{n_{k}^{2}}\sum_{i=1}^{n_{k}}\left\Vert A_{k,i}-B_{k,i}\right\Vert _{\pi}^{2}\leq\sum_{k=0}^{k_{max}}\frac{3}{n_{k}}2^{-2k}.
\end{align*}
And since the expected circuit size of $A_{k,i}$ and $B_{k,i}$ are
at most $2^{\gamma k}$ and $2^{\gamma(k-1)}$
respectively, the expected circuit size is at most 
\[
10\sum_{k=0}^{k_{max}}\left|A_{k,i}\right|+\left|B_{k,i}\right|\leq10(1+2^{-\gamma})\sum_{k=0}^{k_{max}}n_{k}2^{\gamma k}
\]
where the prefactor of $10$ allows us to also capture the cost of
adding up the circuits, because we need to perform at most one sum/substraction
per $A_{k,i}$/$B_{k,i}$ and the number of significant bits that
need to be summed is bounded by the circuit size of $A_{k,i}$ and
$B_{k,i}$. Since a 'full adder' requires two XOR (which themselves
require two AND and one OR), two AND and one OR, we need $c=1+2\cdot3+2+1=10$.
this is obviously very loose in general.

For any $\epsilon>0$, we choose $k_{max}=\left\lceil -\log_{2}\nicefrac{\epsilon}{\sqrt{2}}\right\rceil $
and $n_{k}=\left\lceil \frac{3}{1-2^{-\frac{\gamma-2}{2}}}(\nicefrac{\epsilon}{\sqrt{2}})^{-\frac{\gamma+2}{2}}2^{-\frac{\gamma+2}{2}k}\right\rceil $
to obtain $\mathbb{E}\left\Vert \tilde{f}-f\right\Vert _{\pi}^{2}\leq\epsilon^{2}$
at a computational cost of at most
\begin{align*}
 & 10(1+2^{-\gamma})\sum_{k=0}^{k_{max}}\frac{3}{1-2^{-\frac{\gamma-2}{2}}}\left(\frac{\epsilon}{\sqrt{2}}\right)^{-\frac{\gamma+2}{2}}2^{\frac{\gamma-2}{2}k}+2^{\gamma k}\\
\leq & 10(1+2^{-\gamma})\left[\frac{3\cdot2^{\frac{\gamma+2}{4}}}{1-2^{-\frac{\gamma-2}{2}}}\epsilon^{-\frac{\gamma+2}{2}}\sum_{k=-\infty}^{k_{max}}2^{\frac{\gamma-2}{2}k}+\sum_{k=-\infty}^{k_{max}}2^{\gamma k}\right]\\
\leq & 10(1+2^{-\gamma})\left[\frac{3\cdot2^{\frac{\gamma+2}{4}}}{1-2^{-\frac{\gamma-2}{2}}}\epsilon^{-\frac{\gamma+2}{2}}\frac{2^{\frac{\gamma-2}{2}k_{max}}}{1-2^{-\frac{\gamma-2}{2}}}+\frac{2^{\gamma k_{max}}}{1-2^{\gamma}}\right]\\
\leq & 10(1+2^{-\gamma})\left[\frac{3\cdot2^{\gamma-1}}{\left(1-2^{-\frac{\gamma-2}{2}}\right)^{2}}+\frac{2^{\frac{3}{2}\gamma}}{1-2^{-\gamma}}\right]\epsilon^{-\gamma}.
\end{align*}
\end{proof}
The behavior of the HTMC norm under composition is close but weaker
than subadditivity. For functions $f_{1},\dots,f_{L}$ and their $\epsilon_{\ell}$-approximating
circuits $A_{1},\dots,A_{L}$ we can approximate $f_{L:1}$ by $A_{L:1}$
and a telescopic sum argument gives us
\begin{equation}
C(f_{l}\circ\cdots\circ f_{1},\sum_{\ell}Lip(f_{L:\ell+1})\epsilon_{\ell})\leq\sum_{\ell=1}^{L}C(f_{\ell},\epsilon_{\ell};L_{2}(A_{\ell-1:1}\#\pi)),\label{eq:composition_telescopic_Cfe}
\end{equation}
where $A_{\ell-1:1}\#\pi$ is the pushforward of $\pi$ under $A_{\ell-1:1}$,
i.e. the distribution of $A_{\ell-1:1}(x)$ when $x\sim\pi$. Two
things stand out: first we need to control the Lipschitz constants
of $f_{L:\ell+1}$ to control how the approximation error in the intermediate
steps propagates to the outputs, and second the approximations are made
on the distribution $A_{\ell-1:1}\#\pi$ rather than $f_{\ell-1:1}\#\pi$.
To handle the second problem, one could leverage the fact that $A_{\ell-1:1}$
is an approximation of $f_{\ell-1:1}$ and therefore $f_{\ell-1:1}\#\pi$ and $A_{\ell-1:1}\#\pi$ are close in Wasserstein distance. We will use a simpler strategy in this paper: our bound on
$C(f_{\ell},\epsilon_{\ell};L_{2}(\pi_{\ell}))$ only depends
on the constants $c_{i}$ such that $\left|x_{i}\right|\leq c_{i}$
over the support of $\pi_{\ell}$, we therefore can simply project
the $i$-th output of $A_{\ell-1:1}$ to the range $[-c_{i},c_{i}]$.
This idea is formalized in the following statement:
\begin{prop}
\label{prop:composition_HTMC}Let $f_{1},\dots,f_{L}$ be functions
$f_{\ell}:\mathbb{R}^{d_{\ell-1}}\to\mathbb{R}^{d_{\ell}}$ such that
$\left|f_{\ell:1,i}(x)\right|\leq c_{\ell,i}$ over the support of
$\pi$, and assume $\left\Vert f_{\ell}\right\Vert _{M^{\gamma}(\pi_{\ell-1})}\leq r_{\ell}$
for all distributions $\pi_{\ell-1}$ supported in the hyper-rectangle
$[-c_{\ell-1,1},c_{\ell-1,1}]\times\cdots\times[-c_{\ell-1,d_{\ell}},c_{\ell-1,d_{\ell}}]$,
then
\[
\left\Vert f_{L:1}\right\Vert _{M^{\gamma}(\pi)}^{\frac{\gamma}{\gamma+1}}\leq c^{\frac{1}{\gamma+1}}\sum_{\ell=1}^{L}\left(Lip(f_{L:\ell+1})r_{\ell}\right)^{\frac{\gamma}{\gamma+1}}.
\]
\end{prop}

\begin{proof}
Between each approximation $A_{\ell}$ of $f_{\ell}$, we add a projection
step, that clamps each coordinate to the range $[-c_{\ell,i},c_{\ell,i}]$.
This only reduces the approximation error (since the output $f_{\ell}$
is by assumption already inside this range), and increases the circuit
size by at most a constant prefactor since at each layer a finite number of bits of the representations are changed, and only those need to be projected. Combining Equation \ref{eq:composition_telescopic_Cfe} together
with our assumption on the HTMC norm of $f_{\ell}$, we obtain
\begin{align*}
C(f_{l}\circ\cdots\circ f_{1},\sum_{\ell}Lip(f_{L:\ell+1})\epsilon_{\ell}) & \leq c\sum_{\ell=1}^{L}r_{\ell}^{\gamma}\epsilon_{\ell}^{-\gamma}.
\end{align*}
Optimizing over the choice of $\epsilon_{\ell}$ that result in the
same total error $\epsilon$, we obtain 
\begin{align*}
C(f_{l}\circ\cdots\circ f_{1},\epsilon) & \leq c\left(\sum_{\ell=1}^{L}Lip(f_{L:\ell+1})^{\frac{\gamma}{\gamma+1}}r_{\ell}^{\frac{\gamma}{\gamma+1}}\right)^{\gamma+1}\epsilon^{-\gamma}
\end{align*}
which implies $\left\Vert f_{L:1}\right\Vert _{M^{\gamma}(\pi)}^{\frac{\gamma}{\gamma+1}}\leq c^{\frac{1}{\gamma+1}}\sum_{\ell=1}^{L}\left(Lip(f_{L:\ell+1})r_{\ell}\right)^{\frac{\gamma}{\gamma+1}}$
as needed.
\end{proof}

\subsection{Generalization bound \label{subsec:Generalization_bounds}}

One can rely on the HTMC norm to bound the gap between the risk
\[
R(f)=\mathbb{E}_{X\sim\pi}[\|f(x)-f^*(x)\|^2]=\|f-f^*\|^2_\pi,
\]and empirical risk 
\[\tilde{R}_N(f)=\frac{1}{N}\sum_{i=1}^N\|f(x-i)-f^*(x_i)\|^2]=\|f-f^*\|^2_{\tilde{\pi}_N},\]
for the `true function' $f^*$, the `true distribution' $\pi$, and the empirical measure $\tilde{\pi}_N$ over the dataset $x_1,\dots,x_N$.

The following is a uniform PAC (probably approximately correct) generalization bound, which proves that with a small probability the generalization error will be small over the whole set of HTMC functions:
\begin{thm}
Given a true function $f^*$ that is bounded $\|f^*\|_\infty \leq B$, then with probability $p$, we have that all functions $f$ with bounded $L_\infty$ norm $\|f\|_\infty \leq B$ and bounded $\gamma$-HTMC norm $\left\Vert f\right\Vert _{M_{p=1}^{\gamma}(L_2(\frac{\pi+\tilde{\pi}_{N}}{2}))}$
over the mixed distribution $\frac{\pi+\tilde{\pi}_{N}}{2}$ satisfy
\[
\sqrt{R(f)}-\sqrt{\tilde{R}_{N}(f)} \leq 4\left\Vert f\right\Vert _{M^{\gamma}}\left(\frac{B^2}{\left\Vert f\right\Vert^2_{M^{\gamma}}N}\right)^{\frac{1}{2+\gamma}}\left(1+\sqrt{4\log\left(2\frac{\left\Vert f\right\Vert _{M^{\gamma}}^{2}}{B^{2}}N\right)}\right)+3B\sqrt{\frac{\log\frac{1}{p}}{N}}.
\]
\end{thm}

\begin{proof}
Let us first consider a generic circuit $A$ such that $\left\Vert A\right\Vert _{\infty}\leq B$.
Since the function $h(x)=\left\Vert f^{*}(x)-A(x)\right\Vert ^{2}$
is bounded by $2B^{2}$, we may apply Bernstein inequality to obtain
that with probability $1-p$
\[
R(A)-\tilde{R}_{N}(A)\leq\sqrt{\mathrm{Var}(h(x_{i}))\frac{-2\log p}{N}}+B^{2}\frac{-4\log p}{3N}\leq\sqrt{R(A)B^{2}\frac{-2\log p}{N}}+B^{2}\frac{-4\log p}{3N}
\]
where we used $\mathrm{Var}(h(x_{i}))\leq\mathbb{E}_{\pi}\left[h(x)^{2}\right]\leq2B^{2}\mathbb{E}_{\pi}\left[h(x)\right]=2B^{2}R(A)$.
Using the quadratic formula, we obtain
\begin{align*}
\sqrt{R(A)} & \leq\sqrt{B^{2}\frac{-\log p}{2N}}+\sqrt{B^{2}\frac{-\log p}{2N}+\tilde{R}_{N}(A)+B^{2}\frac{-4\log p}{3N}}\leq\sqrt{\tilde{R}_{N}(A)}+3B\sqrt{\frac{\log\frac{1}{p}}{N}},
\end{align*}
where we used $\frac{1}{\sqrt{2}}+\sqrt{\frac{3+8}{6}}\approx2.0611\leq3$.

Since there are at most $(3T^{2})^{T}$ circuits with $T$ nodes (three
types of nodes each connected to at most 2 other nodes, of which there
are $T$), we can take a union bound over all such circuits and obtain
that with probability $1-p$ we have simultaneously for all circuits of size $T$
\begin{align*}
\sqrt{R(A)}-\sqrt{\tilde{R}_{N}(A)} & \leq3B\sqrt{\frac{T\log 3T^2+\log\frac{1}{p}}{N}}\leq3B\sqrt{\frac{T\log 3T^2}{N}}+3B\sqrt{\frac{\log\frac{1}{p}}{N}}.
\end{align*}

Given a function $f$ with finite HTMC norm $\left\Vert f\right\Vert _{M^{\gamma}}$,
there is for all $\epsilon$ a circuit $A_{\epsilon}$ that $\epsilon$-approximates
it over the sample $S$ over the mixed distribution $\frac{\pi+\tilde{\pi}_{N}}{2}$,
where $\tilde{\pi}_{N}$ is the empirical measure over the training
data. We therefore have $\sqrt{R(f)}-\sqrt{\tilde{R}_{N}(f)}\leq\sqrt{R(A_{\epsilon})}-\sqrt{\tilde{R}_{N}(A_{\epsilon})}+4\epsilon$
so that with probability $1-p$:
\begin{align*}
\sqrt{R(f)}-\sqrt{\tilde{R}_{N}(f)} & \leq\min_{\epsilon}4\epsilon+3B\sqrt{\frac{\left\Vert f\right\Vert _{M^{\gamma}}^{\gamma}\epsilon^{-\gamma}\log(3\left\Vert f\right\Vert _{M^{\gamma}}^{2\gamma}\epsilon^{-2\gamma})}{N}}+3B\sqrt{\frac{\log\frac{1}{p}}{N}}.
\end{align*}
and with $\epsilon=\left(B\frac{\left\Vert f\right\Vert _{M^{\gamma}}^{\frac{\gamma}{2}}}{\sqrt{N}}\right)^{\frac{2}{2+\gamma}}$,
we obtain that $\sqrt{R(f)}-\sqrt{\tilde{R}_{N}(f)}$ is upper bounded by
\begin{align*}
 &4B^{\frac{2}{2+\gamma}}\left(\frac{\left\Vert f\right\Vert _{M^{\gamma}}^{\gamma}}{N}\right)^{\frac{1}{2+\gamma}}\left(1+\sqrt{\log \left(3\left(\frac{\left\Vert f\right\Vert _{M^{\gamma}}}{B}\right)^{\frac{4\gamma}{2+\gamma}} N^{\frac{2\gamma}{2+\gamma}}\right)}\right)+3B\sqrt{\frac{\log\frac{1}{p}}{N}}\\
 & \leq 4\left\Vert f\right\Vert _{M^{\gamma}}\left(\frac{B^2}{\left\Vert f\right\Vert^2_{M^{\gamma}}N}\right)^{\frac{1}{2+\gamma}}\left(1+\sqrt{2\log\left(3\frac{\left\Vert f\right\Vert _{M^{\gamma}}^{2}}{B^{2}}N\right)}\right)+3B\sqrt{\frac{\log\frac{1}{p}}{N}}.
\end{align*}
\end{proof}
Note that if we assume that the train error is zero (or small enough),
this implies a bound on the risk $R(f)$ of order $\left\Vert f\right\Vert _{M^{\gamma}}\left(\frac{B^2}{\left\Vert f\right\Vert^2_{M^{\gamma}}N}\right)^{\frac{2}{2+\gamma}}$
(up to log terms). The LHS of the sandwich bound (Theorem \ref{thm:LHS_sandwich})
allows us to the HTMC norm in terms of th parameter complexity $R(\theta)$,
and since this upper bound only depends on the support of the input
distribution, it can easily be applied to the mixed distribution $\frac{\pi+\tilde{\pi}_{N}}{2}$.

The proof is analogous to bounds based on covering number arguments
\citep{Birman_1967_covering_sobolev,jacot2024_covering_generalization},
where the circuits play the role of the center of the balls of the
covering. However the big advantage is that the set of circuits with
bounded size is independent of the empirical measure $\hat{\pi}_{N}$,
which allows us to apply Bennett's inequality (a type of Chernoff
bound) directly, whereas with covering numbers (or Rademacher complexity) a symmetrization argument is required.

Actually, for many models (described by a family $\mathcal{F}$  of functions) the $\epsilon$-covering number $\mathcal{N}(\mathcal{F},\epsilon)$ is upper bounded (up to log terms) by the max circuit size $\max_{f\in\mathcal{F}} C(f,\epsilon)$. And for many models, e.g. Sobolev functions, this bound is tight, which means that one would recover the same rate with a covering number argument or with a circuit size argument.

One could go even further and consider the $\epsilon$-Kolmogorov complexity $K(f,\epsilon)$ and define a similar Kolmogorov HTMC norm $\|f\|_{K^\gamma}$. A similar generalization bound could then be provenand it would yield rates that match covering arguments on all finite description length models.

\subsection{Vertices of the HTMC ball: Tetrakis functions}\label{subsec:Tetrakis_functions}
Taking a more convex geometric point of view, the previous section proves that the HTMC unit ball $B_{M^\gamma}$ is in some sense pretty small, because we can efficiently apply a uniform bound over it. This could be interpreted as the HTMC ball being `pointy' like the $\ell_1$ ball rather than `boxy' like the $\ell_\infty$ ball.

To make this a little bit more formal, we can try to identify the vertices or extrema points of the HTMC ball, i.e. the smallest subset $V\subset B_{M^\gamma}$ whose convex hull equals the full ball $B_{M^\gamma} = Conv(V)$. Pointy balls have much less vertices than boxy ones, e.g. the $\ell_1$ ball has $2d$ vertices, while the $\ell_\infty$ ball has $2^d$ vertices.

In the context of the HTMC ball we should think of these vertices as \emph{atomic circuits}: circuits that cannot be represented as the average of other simpler circuits. In this section we will identify a set of functions, the Tetrakis functions, that are an approximation for these vertices, in the sense that it is a smaller subset (a countable set, in comparison to the HTMC functions, which are defined as infinite sequence of circuits, and therefore potentially uncountable) whose convex hull is close to the full ball.

And this convex geometric detour will pay off, as these Tetrakis functions will end up playing a central role in our construction bound for the RHS of Theorem \ref{thm:main_sandwich}.

We start from the so-called Tetrakis triangulation of the hyper-grid, which matches the Tetrakis
tiling in 2D. We start from a triangulation of the hyper-cube based on
permutations, where a permutation $\pi\in S_{d_{in}}$ on $d_{in}$
elements corresponds to the simplex
\[
\{x\in[0,1]^{d_{in}}:x_{\pi^{-1}(1)}\leq\cdots\leq x_{\pi^{-1}(d_{in})}\}
\]
which has $d_{in}+1$ vertices $\pi(0,\dots,0),\pi(0,\dots,0,1),\dots,\pi(1,\dots,1)$.
Conversely, any point $x\in[0,1]^{d_{in}}$ is contained in the simplices
corresponding to permutations $\pi$ that sort $x$, i.e. such that
$\pi(x)$ is non-decreasing.

This triangulation is then extended to the whole grid by mirroring
along the faces of the hypercube, leading to a triangulation of the
grid that is invariant under translation by even integer (not odd
ones). We write $V(x)$ for the sequence of $d_{in}+1$ tuples $(v_{i},p_{i})$
of vertices $v$ of the simplex that contains $x$ along with the
weights $p\in[0,1]$ such that $x=\sum_{i=1}^{d_{in}+1}p_{i}v_{i}$.
Note that this decomposition may not be unique, but the non-uniqueness
is only in the vertices $v_{i}$ whose weight $p_{i}$ is $0$.

One of the main parts of the construction will be defining  a neural network
that maps a point $x$ to the weighted binary representations $\left(p_{i}Bin(v_{i})\right)_{i=1,\dots,d_{in}+1}$$\left(p_{i}Bin(v_{i})\right)_{i=1,\dots,d_{in}+1}$.
Note that this map is now not only unique/well-defined but also continuous.

\begin{figure}
    \centering
    \includegraphics[width=0.9\linewidth]{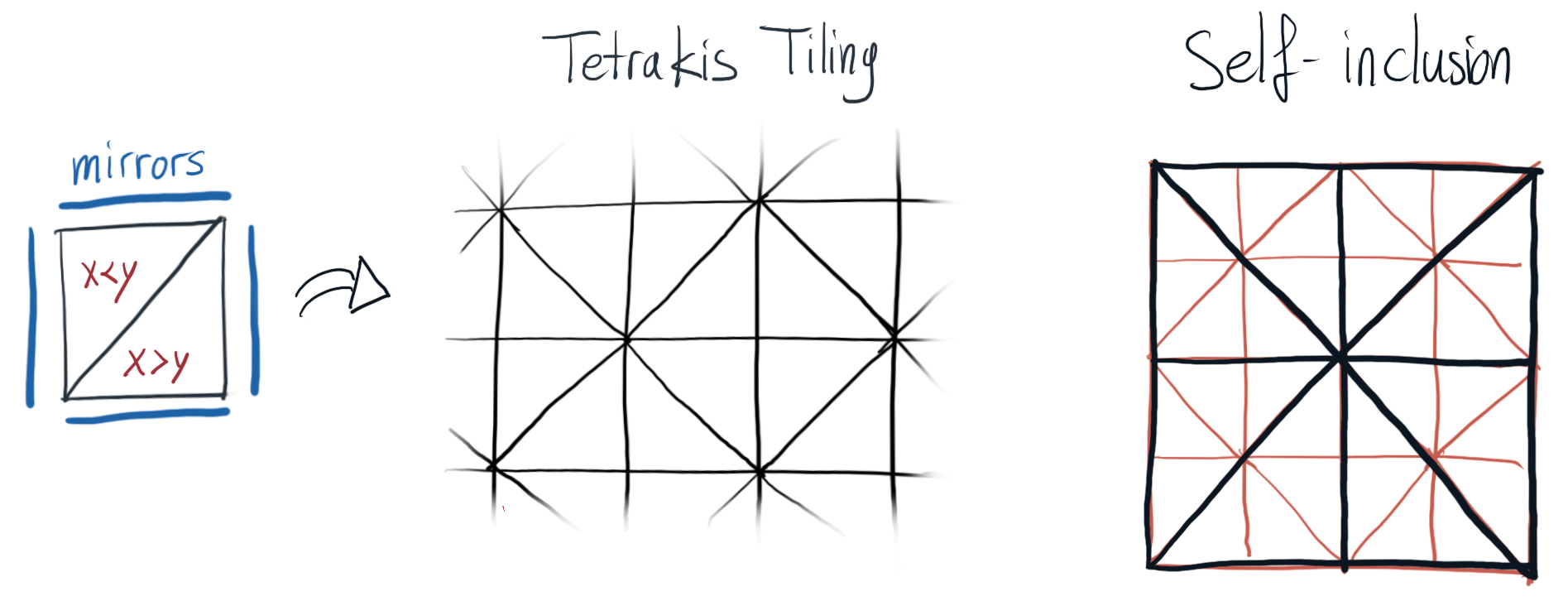}
    \caption{The Tetrakis triangulation in 2D.}
    \label{fig:placeholder}
\end{figure}

Given a circuit $C$ taking $d_{in}M$ binary inputs, i.e. the binary
representations of $d_{in}$ integers in the range $\{0,\dots,2^{M}-1\}$,
and returns $d_{out}M_{out}$ binary outputs, we write $TK_{M}[C]:[0,1]^{d_{in}}\to[0,1]^{d_{out}}$
for the extension of the values of $C$ (which are defined on the grid
$\{n2^{-\frac{M}{2}}:n\in\{0,\dots,2^{\frac{M}{2}}\}\}$) to the hyper-cube
using the Tetrakis triangulation: 
\[
TK_{M}[C](x)=\sum_{(v,p)\in V_{M}(x)}p\sum_{m=1}^{M_{out}}2^{-m}C_{\cdot,m}(Bin(v)).
\]
The set $\mathcal{TK} = \{TK_M[C]:\forall \text{circuits } C, M\in\mathbb{N}\}$ of Tetrakis function is countable since the circuits are themselves countable. Furthermore, these functions are HTMC computable:

\begin{prop}
\label{prop:HTMC_norm_of_TK_functions}There is a universal constant
$c$ such that for all circuits $C$, we have
\[
C_p(TK_{M}[C],\epsilon;L_\infty )\leq\begin{cases}
0 & \epsilon\geq1\\
cd_{in}\left(\left|C\right|_p+M-\log_{2}\epsilon\right) & \epsilon<1.
\end{cases}
\]
and therefore
\[
\left\Vert TK_{M}[C]\right\Vert _{M_p^{\gamma}(L_\infty)}\leq\left[cd_{in}\left(\left|C\right|_p+M\right)\right]^{\frac{1}{\gamma}}.
\]
\end{prop}

\begin{proof}
Without loss of generality, we may assume that the grid points have
integer value, because one could always rescale the inputs $x_{i}$
to $\frac{x_{i}-a_{i}}{w}2^{M}$.

We start by identifying the cube that contains $x$: we compute the
closest even and odd grid coordinates
\begin{align*}
x_{i}^{e} & =2\left\lfloor \frac{x_{i}}{2}\right\rfloor +1, & x_{i}^{o} & =2\left\lfloor \frac{x_{i}-1}{2}\right\rfloor 
\end{align*}
as well as the cube coordinates of $x_{i}$:
\[
s_{i}=\left|x_{i}-x_{i}^{e}\right|\in[0,1]
\]
which equals $1$ when $x_{i}$ equals $x_{i}^{o}$ and $0$ when
it equals $x_{i}^{e}$, and interpolates linearly between them. These
can be bomputes can be done in $O(d_{in}\left(M+\log_{2}\epsilon\right))$
steps.

We now look for the simplex that contains $x$, by sorting the $s_{i}$,
yielding a non-decreasing list $t_{1},\dots,t_{d_{in}}$ to which
we append $t_{0}=0$ and $t_{d_{in}+1}=1$. Given the permutation
$\pi$ that sorts the $s_{i}$, the $d_{in}+1$ vertices of the simplex
are then 
\begin{align*}
(y_{0},\dots,y_{d_{in}}) & =x^{\pi^{-1}(e,\dots,e)},x^{\pi^{-1}(o,e,\dots,e)},\dots,x^{\pi^{-1}(o,\dots,o,e)},x^{\pi^{-1}(o,\dots,o)},
\end{align*}
This can rewritten as 
\[
y_{k,i}=\begin{cases}
x_{i}^{e} & \pi(i)>k\\
x_{i}^{o} & \pi(i)\leq k
\end{cases}=\begin{cases}
x_{i}^{e} & s_{i}\geq t_{k+1}\\
x_{i}^{o} & s_{i}\leq t_{k}
\end{cases}
\]
because $\pi(i)\geq k+1$ implies $s_{i}=t_{\pi(i)}\geq t_{k+1}$
because the $t_{k}$s are non-decreasing, and similarly $\pi(i)\leq k$
implies $s_{i}\leq t_{k}$. The weight of the $k$-th vertex are then
simpy $p_{k}=t_{k+1}-t_{k}$. This can all be done with a circuit
of size $O(d_{in}(\log d_{in})^{2}\log_{2}\epsilon)$: the main cost
is a sorting network with $d_{in}(\log d_{in})^{2}$ comparison each
of cost $\log_{2}\epsilon$.

We then evaluate the circuit $C$ on the vertices and compute the
weighted sum
\[
TK_{M}[C](x)=\sum p_{k}C(y_{k}).
\]
which requires $O(d_{in}\left|C\right|_p+\log_{2}\epsilon)$ circuit size.
There is therefore a constant $c$ such that 
\[
C_p(TK_{M}[C],\epsilon;L_\infty)\leq\begin{cases}
0 & \epsilon\geq1\\
cd_{in}\left(\left|C\right|_p+M-\log_{2}\epsilon \right) & \epsilon<1.
\end{cases}
\]

This also implies that 
\[
\left\Vert TK_{M}[C]\right\Vert _{M_p^{\gamma}(L_\infty)}^{\gamma}\leq\max_{\epsilon\geq1}cd_{in}\left(\left|C\right|_p+M+\log_{2}\epsilon\right)\epsilon^{-\gamma}=cd_{in}\left(\left|C\right|+M\right)
\]
since the maximum is always attained at $\epsilon=1$.
\end{proof}

Now that we have shown that the Tetrakis functions are included in the HTMC ball, we want to show that their convex hull contains the HTMC ball, by showing that HTMC functions can be written as sums of Tetrakis functions. This turns out to be possible if the underlying norm is the $L_\infty$ norm and if we also assume some Hölder continuity:
\begin{prop}
\label{prop:TK_decomposition}Consider a function $f:[0,1]^{d_{in}}\to\mathbb{R}^{d_{out}}$
with bounded HTMC norm $\left\Vert f\right\Vert _{M_{p}^{\gamma}(L_{\infty})}$
for some $\gamma>0$ and bounded Hölder norm $\left\Vert f\right\Vert _{C^{\alpha}}$
for some $\alpha\in(0,1]$, then there are circuits $C_{k}$ for $k\geq k_{min}=-\left\lceil \log_{2}\left\Vert f\right\Vert _{\infty}\right\rceil $
such that $\left|C_{k}\right|_{p}\leq c_{1}\left\Vert f\right\Vert _{M_{p}^{\gamma}(L_{\infty})}^{\gamma}2^{\gamma k}$
and 
\[
\left\Vert f-\sum_{k=k_{min}}^{K}2^{-k+3}TK_{M(k)}[C_{k}]\right\Vert _{\infty}\leq2^{-K}
\]
for $M(k)=\left\lceil \frac{1}{\alpha}\log_{2}\left\Vert f\right\Vert _{C^{\alpha}}+\frac{1}{2}\log_{2}d_{in}+\frac{1}{\alpha}k\right\rceil $.
\end{prop}

\begin{proof}
Let $A_{1},A_{2},\dots$ be the circuits that approximate $f$ in
$L_{\infty}$-norm $\left\Vert f-A_{k}\right\Vert _{\infty}\leq2^{-k}$
over the $2^{-M(k)}$-grid, the Tetrakis extension $TK_{M(k)}[A_{k}]$
for 
\[
M(k)=\left\lceil \frac{1}{\alpha}\log_{2}\left\Vert f\right\Vert _{C^{\alpha}}+\frac{1}{2}\log_{2}d_{in}+\frac{1}{\alpha}k\right\rceil \]
then satisfies
\begin{align*}
\left\Vert f(x)-TK_{k'}[A_{k}](x)\right\Vert  & \leq\left\Vert f(x)-TK_{M(k)}[f](x)\right\Vert +\left\Vert TK_{M(k)}[f](x)-TK_{M(k)}[A_{k}](x)\right\Vert \\
 & \leq\sum_{v,p\in V(x)}p\left(\left\Vert f(x)-f(v)\right\Vert +\left\Vert f(v)-A_{k}(v)\right\Vert \right)\\
 & \leq\left\Vert f\right\Vert _{C^{\alpha}}d_{in}^{\frac{\alpha}{2}}2^{-\alpha M(k)}+2^{-k}\leq2^{-k+1}.
\end{align*}

Thanks to the self-inclusion property of the Tetrakis triangulation, there
is a circuit $B_{k}$ such that $TK_{M(k)+1}[B_{k}]=TK_{M(k)}[A_{k}]$
and with $\left|B_{k}\right|_{p}\leq c_{p}\left|A_{k}\right|_{p}$ for some constant $c_p$: given a gridpoint $x$ of the finer grid, if the $i$-th coordinate $x_i$ is odd, we map it to the closest even and odd coordinates $x_i^e,x_i^o$ of the coarser grid, and if it is even, it already lies on the coarser grid and we set $x_i^e=x_i^o=x_i/2$, we then evaluate $A_k$ on $x_i^e$ and $x_i^o$ and average the outputs.

We can therefore approximate $f$ with a telescopic sum 
\[
\left\Vert f-\sum_{k=1}^{K}TK_{M(k)}\left[A_{k}-B_{k-1}\right]\right\Vert _{\infty}\leq2^{-K}.
\]
with the advantage that the summands are small
\begin{align*}
\left\Vert TK_{M(k)}\left[A_{k}-B_{k-1}\right]\right\Vert _{\infty} & \leq\left\Vert TK_{M(k)}\left[A_{k}\right]-TK_{M(k-1)}\left[A_{k-1}\right]\right\Vert _{\infty}\\
 & \leq\left\Vert TK_{M(k)}\left[A_{k}\right]-f\right\Vert _{\infty}+\left\Vert TK_{M(k-1)}\left[A_{k-1}\right]-f\right\Vert _{\infty}\\
 & \leq2^{-k+1}+2^{-k+2}=6\cdot2^{-k}.
\end{align*}
This implies that in the signed difference $A_{k}-B_{k-1}$, i.e.
all the bits before location $k-\left\lceil \log_{2}6\right\rceil =k-3$
are zero. There is a circuit $C_{k}$ with signed outputs such that $2^{-k+3}TK_{M(k)}[C_{k}]=TK_{M(k)}[A_{k}-B_{k-1}]$
and there is a constant $c_{1}$ such that $\left|C_{k}\right|_{p=\frac{2}{3}}\leq c_{1}\left\Vert f\right\Vert _{M^{\gamma}(\infty,p)}^{\gamma}2^{\gamma k}$
because $C_{k}$ is obtained by a finite number of sums of the algorithms
$A_{k}$ and $A_{k-1}$.
\end{proof}

These two propositions combine into a corollary that illustrate why these Tetrakis functions can be thought of as approximations of the vertices of the HTMC ball intersected with the Hölder ball. More precisely, for all $\gamma,\alpha$, we define a certain rescaling of our set of Tetrakis functions
\[
\mathcal{TK}_{\gamma,\alpha} = \left\{\frac{TK_{M}[C]}{\sqrt[\gamma]{cd_{in}\left|C\right|}\vee2^{\alpha M}}:\forall \text{circuits } C, M\in\mathbb{N}\right\}
\]
\begin{thm}
For all $\gamma>2$, $\delta>0$ and $\alpha\in (0,1]$, there are constants $c_{\gamma,\alpha,\delta},C_{\gamma,\alpha,\delta}$ such that
\[
c_{\gamma,\alpha,\delta} d_{in}^{-(\frac{1}{\gamma} \vee \frac{\alpha}{2})} \left( B_{M^{\gamma-\delta}(L_{\infty})}\cap B_{C^{\alpha}} \right)\subset CConv\;\mathcal{TK}_{\gamma,\alpha} \subset C_{\gamma,\alpha,\delta} \left( B_{M^{\gamma}(L_2(\pi))}\cap B_{C^{\alpha}} \right),
\]
where $CConv(\Omega)$ is the closed convex hull, or the closure of the convex hull.
\end{thm}

\begin{proof}
(RHS) Since the RHS is convex, we only need to prove that the scaled
Tetrakis function are included in the RHS. This follows from the fact
that 
\[
\left\Vert \frac{TK_{M}[C]}{\sqrt[\gamma]{cd_{in}\left|C\right|}\vee2^{\alpha M}}\right\Vert _{M^{\gamma}}\leq\frac{\sqrt[\gamma]{cd_{in}\left(\left|C\right|+M\right)}}{\sqrt[\gamma]{cd_{in}\left|C\right|\vee M}}\leq\sqrt{2}
\]
and 
\[
\left\Vert \frac{TK_{M}[C]}{\sqrt[\gamma]{cd_{in}\left|C\right|}\vee2^{\alpha M}}\right\Vert _{C^{\alpha}}\leq\frac{\left\Vert TK_{M}[C]\right\Vert _{C^{\alpha}}}{2^{\alpha M}}\leq\frac{2^{\alpha M}}{2^{\alpha M}}=1
\]
because at worst two neighboring grid points have a difference in
values of at most $1$ and a distance of $2^{-M}$.

(LHS) Proposition \ref{prop:TK_decomposition} tells us that for any function $f$ with
$\left\Vert f\right\Vert _{M^{\gamma(1-\delta)}},\left\Vert f\right\Vert _{C^{\frac{\alpha}{1-\delta}}}\leq1$,
then there are circuits $(C_{k})_{k\geq k_{min}}$ which define a series of approximations 
\[
\tilde{f}_K-\sum_{k=k_{min}}^{K}2^{-k+3}TK_{M(k)}[C_{k}]
\]
with $M(k)=\left\lceil \frac{1}{2}\log_{2}d_{in}+\frac{1-\delta}{\alpha}k\right\rceil $, $\left|C_{k}\right|_{p}\leq c_{1}2^{\gamma(1-\delta)k}$,
and 
\[
\left\Vert f-\tilde{f}_K\right\Vert _{\infty}\leq2^{-K}.
\]
Our goal is now to show that this sum of Tetrakis is a convex combination:
\begin{align*}
\tilde{f}_K & =\sum_{k=k_{min}}^{K}2^{-k+3}\left(\sqrt[\gamma]{cd_{in}\left|C\right|}\vee2^{\alpha M(k)}\right)\frac{TK_{M(k)}[C_{k}]}{\sqrt[\gamma]{cd_{in}\left|C\right|}\vee2^{\alpha M(k)}}\\
 & =\sum_{k=k_{min}}^{K}2^{-k+3}\left(\sqrt[\gamma]{cd_{in}c_{1}}2^{(1-\delta)k}\vee d_{in}^{\frac{\alpha}{2}}2^{(1-\delta)k}\right)\frac{TK_{M(k)}[C_{k}]}{\sqrt[\gamma]{cd_{in}\left|C\right|}\vee2^{\alpha M(k)}}\\
 & =\sum_{k=k_{min}}^{K}2^{3-\delta k}\left(\sqrt[\gamma]{cd_{in}c_{1}}\vee d_{in}^{\frac{\alpha}{2}}\right)\frac{TK_{M(k)}[C_{k}]}{\sqrt[\gamma]{cd_{in}\left|C\right|}\vee2^{\alpha M(k)}}
\end{align*}
Therefore, there is a constant $c_{\gamma,\alpha,\delta}$ such that  $c_{\gamma,\alpha,\delta} d_{in}^{-(\frac{1}{\gamma} \vee \frac{\alpha}{2})} \tilde{f}_K$ is contained in the convex hull for all $K$.
\end{proof}

This result can be interpreted as saying that the Tetrakis function are an approximation of the vertices of the HTMC ball intersected with the Hölder ball, up to the constants $c_{\gamma,\alpha,\delta},C_{\gamma,\alpha,\delta}$, the log terms which lead to the $\delta$s and the discrepancy between the $L_\infty$ vs $L_2(\pi)$ norms. Also we did not prove that the Tetrakis functions are the smallest such set, only that it is countable and therefore `much smaller' than the full ball. 

This result also allows us to concretize the idea of using convex optimization to minimize circuit size, as it implies that HTMC norm minimization
\[
\min_f \|f-f^*\|_\pi^2 + \lambda (\|f\|_{M^\gamma}+\|f\|_{C^\alpha})
\]
is equivalent (up to constants) to the infinite dimensional LASSO problem
\[
\min_\beta \left\|\sum_{M\in\mathbb{N},\text{circuit }C} \beta_{M,C} \frac{TK_M[C]}{\sqrt[\gamma]{cd_{in}\left|C\right|}\vee2^{\alpha M}} - f^* \right\|_\pi^2 + \lambda \|\beta\|_1,
\]
which could for example be minimized with some form of greedy coordinate descent, thanks to the fact that the Tetrakis functions are countable. Obviously, because there is a exponential number of Tetrakis functions of size larger than some threshold $\nu$, this would still require a super-polynomial runtime in the worst case. But the convexity could still yield some advantage, in particular if the true function is a sum of many simple circuits: a number of steps that is `only' exponential in the size of the largest sub-circuit could be sufficient to obtain a good estimate for the minimal circuit.

This is a very promising direction, but we leave a careful analysis of such an algorithm to follow up work, and instead focus on connecting the HTMC norm to ResNets. The major role that these Tetrakis functions will play in our analysis is a testament to the power of a convex geometric approach to computation in the HTMC regime.

\section{ResNets as a Computation Paradigm}\label{subsec:ResNets}
In this section, we focus on ResNets and their related (pseudo-)norm $\|f\|_{R^\omega}$, and prove our main sandwich bound, Theorem \ref{thm:main_sandwich}, using the theoretical tools developed in the previous sections.

In Section \ref{subsec:Basic-properties-ResNet}, we prove a few basic properties of the ResNet complexity measure and related norm. Section \ref{sec:pruning_bound} then proves the LHS of Theorem \ref{thm:main_sandwich} in the form of a pruning bound. Finally, the RHS is proven in Section \ref{subsec:construction_bound} with a construction bound based on the Tetrakis functions of Section \ref{subsec:Tetrakis_functions}.

\subsection{Basic Properties}\label{subsec:Basic-properties-ResNet}
For ResNet norm, we have
\begin{prop}\label{prop:properties_ResNet_norm}
For two functions $f,g$, and a norm $\left\Vert \cdot\right\Vert $
, we have
\begin{enumerate}
\item If $\left\Vert f\right\Vert \leq\left\Vert f\right\Vert _{\infty}$
for all $f$, then $\left\Vert f\right\Vert \leq\omega\left(\frac{2}{\omega-1}\right)^{\frac{\omega-1}{\omega}}\left\Vert f\right\Vert _{R^{\omega}(\left\Vert \cdot\right\Vert )}$,
\item $\left\Vert f\circ S\right\Vert _{C^{\alpha=\omega^{-1}}}\leq\omega\left(\frac{2}{\omega-1}\right)^{\frac{\omega-1}{\omega}}\left\Vert f\right\Vert _{R^{\omega}(L_{\infty})}$,
\item $\left\Vert f+g\right\Vert _{R^{\omega}}\leq\left(\left\Vert f\right\Vert _{R^\omega}^{\frac{2\omega}{1+2\omega}}+\left\Vert g\right\Vert _{R^\omega}^{\frac{2\omega}{1+2\omega}}\right)^{\frac{1+2\omega}{2\omega}}$,
\item $\left\Vert f\circ g\right\Vert _{R^\omega}\leq2\left(\left\Vert f\right\Vert _{R^\omega}^{\frac{2\omega}{1+2\omega}}+Lip(f)^{\frac{2\omega}{1+2\omega}}\left\Vert g\right\Vert _{R^\omega}^{\frac{2\omega}{1+2\omega}}\right)^{\frac{1+2\omega}{2\omega}}$.
\end{enumerate}
Also note that $\omega\left(\frac{2}{\omega-1}\right)^{\frac{\omega-1}{\omega}}\leq3$
and it converges to $1$ as $\omega\searrow 1$.

\end{prop}

\begin{proof}
(1) $\left\Vert f_{\theta}\right\Vert _{\infty}\leq\sum_{i,j=1}^{d}\left|W_{out,ij}\right|\left\Vert \alpha_{L,j}\right\Vert _{\infty}=\left\Vert W_{out}C_{L}\right\Vert _{1}\leq R(\theta).$

More generally, $\left\Vert f\right\Vert \leq\min_{\epsilon}\left\Vert f_{\theta_{\epsilon}}\right\Vert _{\infty}+\epsilon\leq\min_{\epsilon}\left\Vert f\right\Vert _{R^{\omega}}^{\omega}\epsilon^{1-\omega}+\epsilon=\omega(\omega-1)^{\frac{1-1}{\omega}}\left\Vert f\right\Vert _{R^{\omega}}.$ 

(2) We first observe that $Lip(f_{\theta}\circ S)\leq Lip((\alpha_{0}\to f_{\theta})D_{0}^{-1})\left\Vert D_{0}W_{in}S\right\Vert _{op}\leq\left\Vert D_{0}W_{in}S\right\Vert _{1}\leq R(\theta).$

Now for all $\epsilon>0$, we have parameter $\theta_{\epsilon}$
such that $\left\Vert f-f_{\theta_{\epsilon}}\right\Vert _{\infty}\le\epsilon$
and $R(\theta_{\epsilon})\leq\left\Vert f\right\Vert _{R^{\omega}}^{\omega}\epsilon^{1-\omega}$,
which implies that $Lip(f_{\theta_{\epsilon}}\circ S)\leq\left\Vert f\right\Vert _{R^{\omega}}^{\omega}\epsilon^{1-\omega}$.
Now for any two points $x,y$, we have 
\begin{align*}
\left\Vert f(Sx)-f(Sy)\right\Vert  & \leq\min_{\epsilon}\left\Vert f(Sx)-f_{\theta_{\epsilon}}(Sx)\right\Vert +\left\Vert f_{\theta_{\epsilon}}(Sx)-f_{\theta_{\epsilon}}(Sy)\right\Vert +\left\Vert f_{\theta_{\epsilon}}(Sy)-f(Sy)\right\Vert \\
 & \leq\min_{\epsilon}2\epsilon+\left\Vert f\right\Vert _{R^{\omega}}^{\omega}\epsilon^{1-\omega}\left\Vert x-y\right\Vert \\
 & =\omega\left(\frac{2}{\omega-1}\right)^{\frac{\omega-1}{\omega}}\left\Vert f\right\Vert _{R^{\omega}}\left\Vert x-y\right\Vert ^{\frac{1}{\omega}}.
\end{align*}
This implies $\left\Vert f\circ S\right\Vert _{C^{\alpha=\omega^{-1}}}\leq\omega\left(\frac{2}{\omega-1}\right)^{\frac{\omega-1}{\omega}}\left\Vert f\right\Vert _{R^{\omega}}$. 

(3) We put two networks with parameters $\theta$ and $\phi$ in parallel
$\theta|\phi$ so that $f_{\theta|\phi}=f_{\theta}+f_{\phi}$. Given
$E_{\ell}$ and $F_{\ell}$ the diagonal matrices that yield the optimal
representation costs for $\theta$ and $\phi$ respectively, we choose
$D_{\ell}=\left(\begin{array}{cc}
a^{-\frac{1}{2}}E_{\ell} & 0\\
0 & (1-a)^{-\frac{1}{2}}F_{\ell}
\end{array}\right)$ and obtain
\[
Lip((\alpha_{\ell}\to f_{\theta|\phi})\circ D_{\ell}^{-1})\leq\sqrt{a Lip((\alpha_{\ell}\to f_{\theta})E_{\ell}^{-1})^{2}+(1-a)Lip((\alpha_{\ell}\to f_{\phi})F_{\ell}^{-1})^{2}}\leq1.
\]
We add $\tilde{\cdot}$ to all matrices corresponding to the second
set of parameters $\phi$, and since $\min_{a}a^{-\frac{1}{2}}x+(1-a)^{-\frac{1}{2}}y=\left(x^{\frac{2}{3}}+y^{\frac{2}{3}}\right)^{\frac{3}{2}}$
we obtain
\begin{align*}
R(\theta|\phi)^{\frac{2}{3}} & \leq\sum_{\ell=1}^{L-1}\left(\min_{a_{\ell}}a_{\ell}^{-\frac{1}{2}}\left\Vert E_{\ell}\left|W_{\ell}\right|\text{\ensuremath{\left|V_{\ell}\right|}}C_{\ell}\right\Vert _{1}+(1-a_{\ell})^{-\frac{1}{2}}\left\Vert F_{\ell}\left|\tilde{W}_{\ell}\right|\left|\tilde{V}_{\ell}\right|\tilde{C}_{\ell}\right\Vert _{1}\right)^{\frac{2}{3}}\\
 & \leq\sum_{\ell=1}^{L-1}\left\Vert E_{\ell}\left|W_{\ell}\right|\text{\ensuremath{\left|V_{\ell}\right|}}C_{\ell}\right\Vert _{1}^{\frac{2}{3}}+\left\Vert F_{\ell}\left|\tilde{W}_{\ell}\right|\left|\tilde{V}_{\ell}\right|\tilde{C}_{\ell}\right\Vert ^{\frac{2}{3}}\\
 & =R(\theta)^{\frac{2}{3}}+R(\phi)^{\frac{2}{3}}.
\end{align*}
Note that for simplicity, we have dropped the terms corresponding to
$W_{in}$ and $W_{out}$, but they can be handled similarly.

To extend the bound to $\omega>1$, we choose $\epsilon_{f}=\left(\left\Vert f\right\Vert _{R^{\omega}}^{\frac{2\omega}{1+2\omega}}+\left\Vert g\right\Vert _{R^{\omega}}^{\frac{2\omega}{1+2\omega}}\right)^{-1}\left\Vert f\right\Vert _{R^{\omega}}^{\frac{2\omega}{1+2\omega}}\epsilon$
and $\epsilon_{g}=\left(\left\Vert f\right\Vert _{R^{\omega}}^{\frac{2\omega}{1+2\omega}}+\left\Vert g\right\Vert _{R^{\omega}}^{\frac{2\omega}{1+2\omega}}\right)^{-1}\left\Vert g\right\Vert _{R^{\omega}}^{\frac{2\omega}{1+2\omega}}\epsilon$
and choose parameters $\theta_{\epsilon_{f}}$ and $\phi_{\epsilon_{g}}$
such that $\left\Vert f+g-f_{\theta_{\epsilon_{f}}|\phi_{\epsilon_{g}}}\right\Vert \leq\epsilon_{f}+\epsilon_{g}=\epsilon$
and 
\begin{align*}
R(\theta_{\epsilon_{f}}|\phi_{\epsilon_{g}})^{\frac{2}{3}} & \leq R(\theta_{\epsilon_{f}})^{\frac{2}{3}}+R(\phi_{\epsilon_{g}})^{\frac{2}{3}}\leq\left\Vert f\right\Vert _{R^{\omega}}^{\frac{2\omega}{3}}\epsilon_{f}^{\frac{2}{3}(1-\omega)}+\left\Vert g\right\Vert _{R^{\omega}}^{\frac{2\omega}{3}}\epsilon_{g}^{\frac{2}{3}(1-\omega)}\\
 & =(\left\Vert f\right\Vert _{R^{\omega}}^{\frac{2\omega}{1+2\omega}}+\left\Vert g\right\Vert _{R^{\omega}}^{\frac{2\omega}{1+2\omega}})^{\frac{2}{3}\omega+\frac{1}{3}}\epsilon^{\frac{2}{3}(1-\omega)}
\end{align*}
which implies $\left\Vert f+g\right\Vert _{R^{\omega}}\leq\left(\left\Vert f\right\Vert _{R^{\omega}}^{\frac{2\omega}{1+2\omega}}+\left\Vert g\right\Vert _{R^{\omega}}^{\frac{2\omega}{1+2\omega}}\right)^{\frac{1+2\omega}{2\omega}}$.

(4) We obtain $\left\Vert f\circ g\right\Vert _{R}\leq2\left(\left\Vert f\right\Vert _{R}^{\frac{2}{3}}+Lip(f)^{\frac{2}{3}}\left\Vert g\right\Vert _{R}^{\frac{2}{3}}\right)^{\frac{2}{3}}$
by simply composing the two networks, with the specificity that the
middle $W_{in}$ and $W_{out}$ need to be doubled with opposite signs
to leverage the property of the ReLU that $x=\sigma(x)-\sigma(-x)$,
thus potentially leading to a doubling of the $R$-norm. This constant
factor could be avoided with some small modifications to the architecture.

As for (3), we now optimize over the choice of $\epsilon_{f}$ and $\epsilon_{g}$
and obtain
\[
\left\Vert f\circ g\right\Vert _{R}\leq2\left(\left\Vert f\right\Vert _{R}^{\frac{2\omega}{1+2\omega}}+Lip(f)^{\frac{2\omega}{1+2\omega}}\left\Vert g\right\Vert _{R}^{\frac{2\omega}{1+2\omega}}\right)^{\frac{1+2\omega}{2\omega}}.
\]
\end{proof}
We see the similarity with the HTMC norm, with the notable additional
property that the $R$-norm controls the Lipschitzness of the function,
which is not at all the case for the HTMC norm, and this will result
in the additional Hölder norm term in in the LHS of the main sandwich bound.

\subsection{Pruning Bound} \label{sec:pruning_bound}

We can now prove our pruning bound
\begin{thm}
\label{thm:LHS_sandwich}For any ResNet $f_{\theta}$, one has for
all $\gamma>2$, $\left\Vert f_{\theta}\right\Vert _{M^{\gamma}}\leq c_{\gamma}R(\theta)$.

Furthermore, for all $\omega>1$ and $\delta>0$ we have \[\left\Vert f\right\Vert _{M^{\gamma=2\omega+\delta}}\leq c_{\omega,\delta}\left\Vert f\right\Vert _{R^{\omega}(L_{2}(\pi))}.\]
\end{thm}

\begin{proof}
(1) We write $f_{\theta}=f_{L:1}$ for $f_{\ell}(x)=D_\ell \left(x+W_{\ell}\sigma(V_{\ell}^{T}(x,1))\right)D_{\ell-1}^{-1}$, where we write $(x,1)$ for the vector $x$ with a $1$ concatenated at the end, and the $D_\ell$s are diagonal matrices such that $Lip((\alpha_\ell \to f_\theta) \circ D_\ell^{-1})=Lip(f_{L:\ell+1})\leq 1$. 

Since the outputs of $f_\ell$ will be contained inside the hyper-rectangle with sides $D_\ell C_\ell$, we may use the compositionality of the HTMC norm (Proposition \ref{prop:composition_HTMC}):
\[
\left\Vert f_{\theta}\right\Vert^{\frac{\gamma}{\gamma+1}}_{M^{\gamma}(\pi)}\leq \sum_{\ell=1}^{L}\left\Vert f_\ell \right\Vert^{\frac{\gamma}{\gamma+1}}_{M^{\gamma}(\pi_{\ell-1})},
\]
where $\pi_{\ell}$ is any distribution supported on the hyper-rectangle with sides $D_\ell C_\ell$.

We then use the convexity of the HTMC norm
(Theorem \ref{thm:convexity_HTMC}) to bound the HTMC norm of $f_\ell$
\[
\left\Vert f_{\ell}\right\Vert _{M^{\gamma}(\pi_{\ell-1})}\leq c_{\gamma}\sum_{i=1}^{w}\left\Vert D_\ell W_{\ell,\cdot i}\sigma(V_{\ell,i\cdot}^{T}(D_{\ell-1}^{-1}x,1))\right\Vert _{M^{\gamma}(\pi_{\ell-1})}.
\]
For a single neuron function $D_\ell W_{\ell,\cdot i}\sigma(V_{\ell,i\cdot}^{T}(D_{\ell-1}^{-1}x,1))$,
the intermediate representation is bounded $\left|V_{\ell,i\cdot}^{T}(D_{\ell-1}^{-1}x,1)\right|\leq\sum_{j=1}^{d+1}\left|V_{\ell,ij}\right|C_{\ell-1,jj}$,
we can therefore split it into functions $g_i(x)=V_{\ell,i\cdot}^{T}(D_{\ell-1}^{-1}x,1)$
and $h_i(z)=D_\ell W_{\ell,i}\sigma(z)$. For all $\gamma>0$, there is a constant
$b_{\gamma}$ such that the HTMC norm of scalar multiplication is bounded
$\left\Vert x\mapsto ax\right\Vert _{M^{\gamma}(\pi)}\leq b_{\gamma}\left|a\right|\left|c\right|$
for any distribution supported in $[-c,c]$. Similarly for the ReLU followed by a multiplication, we have$\left\Vert x\mapsto a\sigma(x)\right\Vert _{M^{\gamma}(\pi)}\leq b_{\gamma}\left|a\right|\left|c\right|$.
We can now bound the HTMC norms of $g$ and $h$ using Theorem \ref{thm:convexity_HTMC}:
\begin{align*}
\left\Vert g_i\right\Vert _{M^{\gamma}(\pi_{\ell-1})} & =\left\Vert x\mapsto\sum_{j=1}^d V_{\ell,ij}D_{\ell-1,jj}^{-1}x_{j}+V_{\ell,d+1}\right\Vert _{M^{\gamma}(\pi_{\ell-1})} \\ 
&\leq c_{\gamma}b_{\gamma}\sum_{j}\left|V_{\ell,ij}\right|C_{\ell-1,jj}\\
\left\Vert h_i\right\Vert _{M^{\gamma}(\pi_{\ell-1})} & \leq\left\Vert z\mapsto\sum_{k=1}^{d}D_{\ell,kk}W_{\ell,ki}\sigma(z)\right\Vert _{M^{\gamma}(g\#\pi_{\ell-1})}\\ 
&\leq c_{\gamma}b_{\gamma}\left(\sum_{k}D_{\ell,kk}\left|W_{\ell,ki}\right|\right)\left(\sum_{j}\left|V_{\ell,ij}\right| C_{\ell-1,jj}\right)
\end{align*}
Since we also have $Lip(h)\leq\left\Vert D_\ell W_{\ell,\cdot i}\right\Vert _{1}$,
we obtain
\[
\left\Vert x\mapsto W_{\ell,\cdot i}\sigma(V_{\ell,i\cdot}^{T}(x,1))\right\Vert _{M^{\gamma}(\pi_{\ell-1})}\leq c^{\frac{1}{\gamma}}2^{\frac{\gamma+1}{\gamma}}c_{\gamma}b_{\gamma} D_\ell \left\Vert W_{\ell,\cdot i}\right\Vert _{1}\left\Vert V_{\ell,i\cdot}C_{\ell-1}\right\Vert _{1}
\]
 and therefore 
\[
\left\Vert f_{\ell}\right\Vert _{M^{\gamma}(\pi_{\ell-1})}\leq c^{\frac{1}{\gamma}}2^{\frac{\gamma+1}{\gamma}}c_{\gamma}^{2}b_{\gamma}\left\Vert \left|W_{\ell}\right|\left|V_{\ell}\right|C_{\ell-1}\right\Vert _{1}.
\]

Applying Proposition \ref{prop:composition_HTMC} to the layer-wise
composition to obtain
\[
\left\Vert f_{\theta}\right\Vert _{M^{\gamma}(\pi)}\leq c^{\frac{2}{\gamma}}2^{\frac{\gamma+1}{\gamma}}c_{\gamma}^{2}b_{\gamma}\left(\sum_{\ell}\left\Vert \left|W_{\ell}\right|\left|V_{\ell}\right|C_{\ell-1}\right\Vert _{1}^{\frac{\gamma}{\gamma+1}}\right)^{\frac{\gamma+1}{\gamma}}
\]

(2) By the definition of the $R^{\omega}(L_{2}(\pi))$ norm, we have
parameters $\theta_{\epsilon}$ for all $\epsilon$ such that $\left\Vert f-f_{\theta_{\epsilon}}\right\Vert _{\pi}\leq\epsilon$
and $R(\theta_{\epsilon})\leq\left\Vert f\right\Vert _{R^{\omega}(L_{2}(\pi))}^{\omega}\epsilon^{1-\omega}$.
On the other hand, the first part of the proof implies that for all
$\gamma>2$ there is a circuit $A_{\epsilon}$ with error $\left\Vert A_{\epsilon}-f_{\theta_{\epsilon}}\right\Vert _{\pi}\leq\epsilon$
and size $\left|A_{\epsilon}\right|\leq C_{\gamma}^{\gamma}R(\theta_{\epsilon})^{\gamma}\epsilon^{-\gamma}$
for $C_{\gamma}=c^{\frac{2}{\gamma}}2^{\frac{\gamma+1}{\gamma}}c_{\gamma}^{2}b_{\gamma}$.
Putting the two together, we have
\[
C(f,2\epsilon)\leq C_{\gamma}^{\gamma}\left\Vert f\right\Vert _{R^{\omega}(L_{2}(\pi))}^{\omega\gamma}\epsilon^{-\gamma\omega}
\]
which implies that $\left\Vert f\right\Vert _{M^{\gamma\omega}}\leq2C_{\gamma}^{\frac{1}{\omega}}\left\Vert f\right\Vert _{R^{\omega}(L_{2}(\pi))}$
and if one chooses $\gamma=2+\omega^{-1}\delta$, we obtain
\[
\left\Vert f\right\Vert _{M^{\gamma=2\omega+\delta}}\leq c_{\omega,\delta}\left\Vert f\right\Vert _{R^{\omega}(L_{2}(\pi))}.
\]
\end{proof}
This proof was written with a heavy use of Theorem \ref{thm:convexity_HTMC}
and Proposition \ref{prop:composition_HTMC}, to showcase the versatility
of these results, but it does obfuscate what is happening concretely.
The same result could also be proven as a pruning bound, where the weights $W_{\ell,ki}$ and $V_{\ell,ij}$ are removed with probabilities proportional to $D_{\ell,kk}\left\Vert V_{\ell,i\cdot}C_{\ell-1}\right\Vert _{1}$
and $\left\Vert D_{\ell}W_{\ell,\cdot i}\right\Vert _{1}C_{\ell-1,jj}$
respectively, combined with some MLMC techniques.

\subsection{Construction Bound} \label{subsec:construction_bound}
We have shown in Section \ref{subsec:Tetrakis_functions} how any function $f$ that is both  HTMC computable and Hölder continuous can be represented as a sum of $O(\log_2\epsilon)$ Tetrakis functions. This implies that if ResNets can represent these Tetrakis function then they can represent general HTMC functions.

We will construct a ResNet that roughly follows the structure of the proof of Proposition \ref{prop:HTMC_norm_of_TK_functions}: first the
input $x$ is mapped to the weighted binary representations of its
surrounding vertices $\left(p_{i}Bin(v_{i})\right)_{i=1,\dots,d_{in}+1}$
and second the circuit is evaluated on each of the $d_{in}+1$ vertices
in parallel before being summed up. The first part can itself be decomposed
into two parts, the first of which computes the binary representations
of each coordinate using an iterative method, and the second part
implements a sorting algorithm to recover the simplex that contains
$x$.

For the second part, we rely on the following result that describes
how to represent a binary circuit as min/max circuit:
\begin{prop}
\label{prop:ResNet_repr_circuit}Given a circuit $C$,  there is
a network with parameters $\theta$ that represents the circuit in
the sense that for any input $x\in\{-a,a\}^{d_{in}}$, we have $f_{\theta}(x)=aC(\frac{x}{a})$
(where the binary values are represented as $\pm1$ instead of the
usual $0,1$). Furthermore, if $\left\Vert x\right\Vert _{\infty}\leq b$
over the input domain, we have
\begin{align*}
R(\theta) & \leq c\sqrt{d_{out}}b\left|C\right|_{p=\frac{2}{3}}.
\end{align*}
\end{prop}

\begin{proof}
We simply replace AND and OR with min and max respectively and NOT
with negation. Each of these operations can be represented with a
finite number of neurons, and we put them together into a neural network
that is sparse, i.e. there is a $c>1$ such that there are at most
$c\left|C\right|$ nonzero entries and all entries have absolute value
at most $c$. Note that since we are considering a ResNet, the 'unused'
variables do not simply vanish, however, with a wide enough network
there is always enough additional room that we can wait for the last
layer to drop those unused variables.

Let us now compute the Lipschitz constants $Lip(f_{L:\ell})$, which
takes the intermediate representations $z$ and computes the outputs
as a mix of max, min, and negations. The function $f_{L:\ell}$ is
piecewise linear, with each 'piece' corresponding to inputs $z$ where
all max and min are attained at only one of these entries. Therefore
the gradient $\nabla(f_{L:\ell})_{k}$ of the $k$-th output has only
one non-zero entry that is either $1$ or $-1$. We therefore have
that $Lip(f_{L:\ell})\leq\sqrt{d_{out}}$.

Since the activations $\alpha_{j}^{(\ell)}(x)$ are bounded by $\left\Vert x\right\Vert _{\infty}\leq b$,
we obtain that 
\[
\left(\sum_{\ell=1}^{L-1}\left(Lip(f_{L:\ell})b\sum_{i}\left\Vert W_{\ell,\cdot i}\right\Vert _{1}\left\Vert V_{\ell,i\cdot}\right\Vert _{1}\right)^{\frac{2}{3}}\right)^{\frac{3}{2}}\leq c\sqrt{d_{out}\mathbb{E}_{x\sim\pi}\left\Vert x\right\Vert _{\infty}^{2}}\left|C\right|_{p=\frac{2}{3}}.
\]
\end{proof}

We can describe the ResNet representation of Tetrakis functions:
\begin{prop}
\label{prop:TK_representation}Given a circuit $C:\{0,1\}^{d_{in}\times M}\mapsto\{0,1\}^{d_{out}\times M_{out}}$,
then
\[
\left\Vert TK_M[C]\right\Vert _{R}=O\left(\sqrt{d_{out}}\left(d_{in}\sqrt{M}2^{M}+d_{in}(\log d_{in})^{2}+\sqrt{(d_{in}+1)}\left|C\right|_{p=\frac{2}{3}}\right)\right).
\]
\end{prop}

\begin{proof}
The construction of the ResNet will be in three parts $f=f_{3}\circ f_{2}\circ f_{1}$
going through dimensions $d_{in}$, $d_{in}(2k+1)$, $(d_{in}+1)d_{in}k$
and finally $d_{out}$. First $f_{1}$ takes an input $x$ and computes
for each coordinate $x_{i}$ the weighted binary representations of
the closest even and odd grid coordinate $(1-s_{i})Bin_{k}(x_{i}^{e}),s_{i}Bin_{k}(x_{i}^{o})$
together with the weight $s_{i}\in[0,1]$. Second,
$f_{2}$ takes these triples for each coordinate and computes the
$d_{in}+1$ weighted binary representations $\left(p_{i}Bin(v_{i})\right)_{i=1,\dots,d_{in}+1}$
of the vertices $v_{i}$ of the simplex that contains $x$. Finally,
$f_{3}$ applies the circuit in parallel to each of these representations
before summing them up.

\textbf{1.} The first function is applied in parallel to each coordinate.
And for each coordinate $x_{i}\in[a_{i},a_{i}+w]$ we use recursion
on the scale of the grid, we first map $x_{i}$ to the weight $s_{1,i}=\frac{x_{i}-a_{i}}{w}$
and the weighted even and odd representations $r_{i,1,1}^{e}=s_{1,i}-1$
and $r_{i,1,1}^{o}=s_{1,i}$ (there is a single even grid point at
$a_{i}$ and one odd one at $a_{i}+w$ corresponding to the integer
$0$ and $1$ respectively, with binary representations $-1$ and
$1$).

For the recursion, we assume that we are given $s_{i,m}\in[0,1]$,
$r_{i,m}^{e}=(1-s_{i,m})Bin_{m}(x_{i,m}^{e})$ and $r_{i,m}^{o}=(1-s_{i,m})Bin_{m}(x_{i,m}^{o})$
for $1\leq m<k$. First note that $s_{i,m+1}=2\min\{s_{i,m},1-s_{i,m}\}$
$x_{i,m+1}^{e}=\begin{cases}
2x_{i,m}^{e} & \text{ if }s_{i,m}\leq0.5\\
2x_{i,m}^{o} & \text{ if }s_{i,m}\geq0.5
\end{cases}$ (the jump at $s_{i,m}=0.5$ does not matter, because all things that
depend on $x_{i,m+1}^{e}$ are multiplied $(1-s_{i,m+1})$ which equals
$0$ at the jump) and $x_{i,m+1}^{o}=x_{i,m}^{e}+x_{i,m}^{o}$. Let
us first define a helper function $g:\{(a,b):\left|a\right|+\left|b\right|=1\}\to\mathbb{R}^{4}$
whose values on the $\ell_{1}$ unit ball interpolate linearly between
8 points located at the corner of the ball and the midpoints between
corners, taking values:
\begin{align*}
g(+1,0) & =(+1,-1;0,0) & g(+\frac{1}{2},+\frac{1}{2}) & =(0,0;+1,-1)\\
g(0,+1) & =(+1,-1;0,0) & g(-\frac{1}{2},+\frac{1}{2}) & =(0,0;-1,+1)\\
g(-1,0) & =(-1,-1;0,0) & g(-\frac{1}{2},-\frac{1}{2}) & =(0,0;-1,-1)\\
g(0,-1) & =(-1,-1;0,0) & g(+\frac{1}{2},-\frac{1}{2}) & =(0,0;-1,+1).
\end{align*}
You should think of $g$ as taking a mix of two binary numbers $(1-s)a$
and $sb$ for $a,b\in\{-1,1\}$, there are then two cases: 
\begin{itemize}
\item If $s=\frac{1}{2}$, $g$ returns two pairs, with the first pair being
a binary representation with zero weight, and the second pair a binary
representation of $a+b$ with weight $1$.
\item If $s=0,1$, $g$ returns two pairs, with the first being a binary
representation of $2a$ or $2b$ (depending on whether $s=0$ or $s=1$)
with weight $1$, and a second binary pair with weight 0.
\end{itemize}
The homogeneous extension of $g$ from the unit disk to the plane
can be represented by a finite number of neurons organized in at most
two layers.

We may now compute $r_{i,m+1}^{e},r_{i,m+1}^{o}$ according to the
formula
\begin{align*}
r_{i,m+1,j}^{e} & =g_{1}(r_{i,m,j}^{e},r_{i,m,j}^{o})\vee g_{2}(r_{i,m,j-1}^{e},r_{i,m,j-1}^{o})\\
r_{i,m+1,j}^{o} & =g_{3}(r_{i,m,j}^{e},r_{i,m,j}^{o})\vee g_{4}(r_{i,m,j-1}^{e},r_{i,m,j-1}^{o}),
\end{align*}
and 
\begin{align*}
r_{i,m+1,m+1}^{e} & =g_{2}(r_{i,m,m}^{e},r_{i,m,m}^{o})\\
r_{i,m+1,m+1}^{o} & =g_{4}(r_{i,m,m}^{e},r_{i,m,m}^{o}),
\end{align*}
Note that since $g_{2}(a,b)\leq g_{1}(a,b)$, one could remove it
from the formula. The maximum $\vee$ here plays the role of a sum
of two binary numbers, and thankfully we never have to sum two $1$s
so we are fine: indeed in the first case $g_{2}$ always outputs $-1$,
and in the second case $g_{3}$ is $1$ if both $Bin(x_{i,m}^{e})_{j}$
and $Bin(x_{i,m}^{o})_{j}$ are $1$ and $g_{4}$ is $1$ if $Bin(x_{i,m}^{e})_{j-1}\neq Bin(x_{i,m}^{o})_{j-1}$,
but since $x_{i,m}^{e}$ and $x_{i,m}^{o}$ are adjacent integer,
if they differ at one digit, they must differ at all subsequent digits,
i.e if $Bin(x_{i,m}^{e})_{j-1}\neq Bin(x_{i,m}^{o})_{j-1}$ then $Bin(x_{i,m}^{e})_{j}\neq Bin(x_{i,m}^{o})_{j}$
too.

We can therefore implement this recursion as a ResNet with parameters
$\theta_{1}$ to compute $(s_{i,k},r_{i,k,\cdot}^{e},r_{i,k,\cdot}^{o})$.
In general the map from the level $m$ grid $(s_{i,m},r_{i,m,\cdot}^{e},r_{i,m,\cdot}^{o})$
to the level $k$ grid $(s_{i,k},r_{i,k,\cdot}^{e},r_{i,k,\cdot}^{o})$
is $O(\sqrt{k}2^{k-m})$ -Lipschitz, because each $r_{i,k}^{e}$ depends
only one of the $r_{i,m,j}^{e}$s or $r_{i,m,j}^{o}$s (the one that
attains all the maxima), and this dependence is of order $2^{k-m}$.
Furthermore, since all activations are of order $1$, we have 
\[
R(\theta_{1})=O\left(d_{in}\sum_{m=1}^{k}\sqrt{k}2^{k-m}\right)=O\left(\sqrt{d_{in}k}2^{k}\right).
\]

\textbf{2.} For the second part, we are given the coordinate weights
$s_{i}$ and weighted representations $r_{i}^{e},r_{i}^{o}$ of the
closest even and odd grid coordinate (we drop the $k$ index $s_{i,k},r_{i,k}^{e},r_{i,k}^{o}$
in the previous part of the proof, since we are only working on the
finest grid), and our goal is to compute the weighted binary representations
of the vertices of the simplex containing $x$. That is for all $m=0,\dots,d_{in}$
and $i=1,\dots,d_{in}$ we need to compute the weights $p_{m}=(s_{\pi^{-1}(m+1)}-s_{\pi^{-1}(m)})$for
each vertices and the weighted binary representations:
\[
A_{m,i}=(s_{\pi^{-1}(m+1)}-s_{\pi^{-1}(m)})\begin{cases}
Bin(x_{i}^{e}) & m\geq\pi(i)\\
Bin(x_{i}^{o}) & m\leq\pi(i)-1
\end{cases}
\]
where $\pi$ is the permutation that sorts the $s_{i}$s, and we we
define $s_{\pi^{-1}(0)}=0$ and $s_{\pi^{-1}(d_{in}+1)}=1$ so that
all differences $(s_{\pi^{-1}(m+1)}-s_{\pi^{-1}(m)})$ are positive
and sum up to 1.

We first sort the $s_{1},\dots,s_{d_{in}}$ into another list $t_{1},\dots,t_{d}$
(i.e. $t_{m}=s_{\pi^{-1}(m)}$), which can be done with a 'sorting
network' where any two pair values $x,y$ are sorted using the formula
$(x\wedge y,x\vee y)$. This requires a $O((\log d_{in})^{2})$ depth
and $O(d_{in}(\log d_{in})^{2})$ comparisons. We then also define
$t_{0}=0$ and $t_{d_{in}+1}=1$.

Now note that because $\pi$ sorts the $s_{i}$, if $m\geq\pi(i)$
then $t_{m}\geq s_{i}$ and if $m\leq\pi(i)-1$ then $t_{m+1}\leq s_{i}$,
we can therefore rewrite
\begin{align*}
A_{m,i} & =(t_{m+1}-t_{m})\begin{cases}
Bin(x_{i}^{e}) & t_{m}\geq s_{i}\\
Bin(x_{i}^{o}) & t_{m+1}\leq s_{i}.
\end{cases}
\end{align*}
This can be rewritten
\[
A_{m,i}=\left(t_{m+1}-[s_{i}]_{t_{m}}^{t_{m+1}}\right)Bin(x_{i}^{e})+\left([s_{j}]_{t_{i}}^{t_{i+1}}-t_{i}\right)Bin(x_{j}^{o})
\]
 in terms of the clamp function $[x]_{a}^{b}=(x\vee a)\wedge b$,
using the fact that 
\begin{align*}
\left(t_{m+1}-[s_{i}]_{t_{m}}^{t_{m+1}}\right) & =\begin{cases}
t_{m+1}-t_{m} & t_{m}\geq s_{i}\\
0 & t_{m+1}\leq s_{i}
\end{cases}\\
\left([s_{i}]_{t_{m}}^{t_{m+1}}-t_{m}\right) & =\begin{cases}
0 & t_{m}\geq s_{i}\\
t_{m+1}-t_{m} & t_{m+1}\leq s_{i}
\end{cases}.
\end{align*}
Writing $d_{m,i}^{e}=\left(t_{m+1}-[s_{i}]_{t_{m}}^{t_{m+1}}\right)$
and $d_{m,i}^{e}=\left([s_{i}]_{t_{m}}^{t_{m+1}}-t_{m}\right)$, the
$A_{ij}$ can now be expressed in terms of the weighted binary representations
$r_{i}^{e}=(1-s_{i})Bin(x_{i}^{e})$ and $r_{i}^{o}=s_{i}Bin(x_{i}^{e})$
instead
\begin{align*}
A_{m,i} & =[r_{i}^{e}]_{-d_{m,i}^{e}}^{+d_{m,i}^{e}}+[r_{i}^{o}]_{-d_{m,i}^{o}}^{+d_{m,i}^{o}}
\end{align*}
where we used the fact that $d_{m,i}^{e}\leq(1-s_{i})$ and $d_{m,i}^{o}\leq s_{i}$.

We can build a ResNet layers with parameters $\theta_{2}$ that represent
this second part of the function, with $R(\theta_{2})=O(kd_{in}(\log d_{in})^{2})$.

\textbf{3.} Finally we have to evaluate the circuit $C:\{0,1\}^{d_{in}\times k}\mapsto\{0,1\}^{d_{out}\times m}$
in parallel to $A_{0},\dots,A_{d_{in}}$ and sum up the outputs. We
use Proposition \ref{prop:ResNet_repr_circuit} to convert $C$ into
a neural network that is repeated in parallel $d_{in}+1$ times, yielding
$B_{0},\dots,B_{d_{in}}\in\mathbb{R}^{d_{out}\times m}$. Since $A_{i}=w_{i}Bin(v_{i})$
for weights $w_{0},\dots,w_{d_{in}}\geq0$ that sum up to one, and
for vectors $v_{0},\dots,v_{d_{in}}$, we know that $B_{i}=w_{i}C(Bin(v_{i}))$.
The final outputs of the network are then obtained as 
\[
b\sum_{j=1}^{m}2^{-j}\left(1+\sum_{i=0}^{d_{in}}B_{i,\cdot,j}\right)=\sum_{j=1}^{m}2^{-j+1}\sum_{i=0}^{d_{in}}w_{i}C\left(\frac{Bin(v_{i})+1}{2}\right)\in[0,b]^{d_{out}}.
\]

The parameters $\theta_{3}$ of this last part will satisfy
\[
R(\theta_{3})\leq C\sqrt{(d_{in}+1)d_{out}}\left|C\right|_{p=\frac{2}{3}}.
\]
Overall, the parameters $\theta$ will satisfy 
\begin{align*}
R(\theta) & \leq R(\theta_{1})Lip(f_{\theta_{3}}\circ f_{\theta_{2}})+R(\theta_{2})Lip(f_{\theta_{3}})+R(\theta_{3})\\
 & =O\left(d_{in}\sqrt{k}2^{k}\sqrt{d_{out}}+kd_{in}(\log d_{in})^{2}\sqrt{d_{out}}+\sqrt{(d_{in}+1)d_{out}}\left|C\right|_{p=\frac{2}{3}}\right)\\
 & =O\left(\sqrt{d_{out}}\left(d_{in}\sqrt{k}2^{k}+kd_{in}(\log d_{in})^{2}+\sqrt{(d_{in}+1)}\left|C\right|_{p=\frac{2}{3}}\right)\right)
\end{align*}
\end{proof}

The RHS of Theorem \ref{thm:main_sandwich} then follows directly from the representation of HTMC functions as sums of Tetrakis function of Proposition \ref{prop:TK_decomposition} from Section \ref{subsec:Tetrakis_functions} along with some of the basic properties of the ResNet norm from Section \ref{subsec:Basic-properties-ResNet}:

\begin{thm} \label{thm:RHS_sandwich}
For any $\omega>1$, and any $\delta>0$, we define $\gamma=\omega-\delta$
and $\alpha=\frac{1}{\omega-\delta}$, then there is a constant $c_{\omega,\delta}$
such that 
\[
\left\Vert f\right\Vert _{R^{\omega}}^{\omega}\lesssim\sqrt{\frac{1}{\alpha}d_{out}d_{in}}\left(d_{in}\left\Vert f\right\Vert _{C^{\alpha}}^{\frac{1}{\alpha}}+\left\Vert f\right\Vert _{M_{p=\frac{2}{3}}^{\gamma}(L_{\infty})}^{\gamma}\right),
\]
where $D$ is the diagonal matrix with diagonal entries equal to the
sidelengths of the hyper-rectangle that contains the input distribution.
\end{thm}

\begin{proof}
Proposition \ref{prop:TK_decomposition} guarantees that $\tilde{f}_{K}=\sum_{k=k_{min}}^{K}2^{-k+3}TK_{M(k)}[C_{k}]$
approximates $f$ within an $2^{-K}$ error. Proposition \ref{prop:TK_representation}
then gives us
\begin{align*}
\left\Vert \tilde{f}_{K}\right\Vert _{R} & \leq\sqrt{K-k_{min}}\sum_{k=k_{min}}^{K}\left\Vert 2^{-k+3}TK_{M(k)}[C_{k}]\right\Vert _{R}\\
 & \lesssim\sqrt{d_{out}\log_{2}\frac{\left\Vert f\right\Vert _{\infty}}{\epsilon}}\sum_{k=k_{min}}^{K}2^{-k+3}\left(d_{in}\sqrt{M(k)}2^{M(k)}+d_{in}(\log d_{in})^{2}+\sqrt{d_{in}+1}\left|C\right|_{p=\frac{2}{3}}\right)\\
 & \lesssim\sqrt{d_{out}d_{in}\log_{2}\frac{\left\Vert f\right\Vert _{\infty}}{\epsilon}}\sum_{k=k_{min}}^{K}d_{in}\left\Vert f\right\Vert _{C^{\alpha}}^{\frac{1}{\alpha}}2^{(\frac{1}{\alpha}-1)k}+\left\Vert f\right\Vert _{M_{p}^{\gamma}(L_{\infty})}^{\gamma}2^{(\gamma-1)k}\\
 & \lesssim\sqrt{\frac{1}{\alpha}d_{out}d_{in}}\left(d_{in}\left\Vert f\right\Vert _{C^{\alpha}}^{\frac{1}{\alpha}}+\left\Vert f\right\Vert _{M_{p}^{\gamma}(L_{\infty})}^{\gamma}\right)\epsilon^{-(\omega-\delta-1)}\log_{2}\frac{\left\Vert f\right\Vert _{\infty}}{\epsilon}.
\end{align*}
This then implies that 
\begin{align*}
\left\Vert f\right\Vert _{R^{\omega}}^{\omega} &\leq\sqrt{\frac{1}{\alpha}d_{out}d_{in}}\left(d_{in}\left\Vert f\right\Vert _{C^{\alpha}}^{\frac{1}{\alpha}}+\left\Vert f\right\Vert _{M_{p}^{\gamma}(L_{\infty})}^{\gamma}\right)\left\Vert f\right\Vert _{\infty}^{\delta}\max_{\epsilon\leq\left\Vert f\right\Vert _{\infty}}\left(\frac{\left\Vert f\right\Vert _{\infty}}{\epsilon}\right)^{-\delta}\log_{2}\frac{\left\Vert f\right\Vert _{\infty}}{\epsilon} \\ &\lesssim\sqrt{\frac{1}{\alpha}d_{out}d_{in}}\left(d_{in}\left\Vert f\right\Vert _{C^{\alpha}}^{\frac{1}{\alpha}}+\left\Vert f\right\Vert _{M_{p}^{\gamma}(L_{\infty})}^{\gamma}\right)\left\Vert f\right\Vert _{\infty}^{\delta}.
\end{align*}
And since $\left\Vert f\right\Vert _{\infty}^{\delta}\leq\left\Vert f\right\Vert _{R^{\omega}}^{\delta}$,
we obtain
\[
\left\Vert f\right\Vert _{R^{\omega}}\lesssim\left(\frac{1}{\alpha}d_{out}d_{in}\right)^{\frac{1}{\omega-\delta}}\left(d_{in}^{\frac{1}{\omega-\delta}}\left\Vert f\right\Vert _{C^{\alpha}}+\left\Vert f\right\Vert _{M_{p}^{\gamma}(L_{\infty})}\right).
\]
\end{proof}

\section{Conclusion}

This paper establishes a connection between two widely distinct models of
real-valued computations: the discrete paradigm of binary computation, and the
continuous paradigm of DNN computation. This opens the door to a whole
range of similar results for different neural network architectures,
or different notions of parameter complexity.

Looking forward, this is just the first step towards proving how DNNs
implicitly minimize circuit size. To achieve this, we will need to
study the training dynamics of DNNs under gradient descent, which
are notoriously hard to capture. I believe that the convexity of the
HTMC regime in function space will be key towards this goal, because
most previous DNN convergence proofs follow from the emergence of
a 'hidden' limiting convexity \citep{Chizat2018,jacot2018neural}.
Some of my preliminary inquiries are promising: it appears that the
function space convexity translates into an asymptotic absence of
barriers in parameter space, which suggests that the main difficulty
will lie in proving that gradient descent manages to escape saddles.

In comparison to a brute force search, we can expect ResNet to have a similar advantage as applying Frank-Wolfe on the Tetrakis functions: if the optimal
circuit can be written as a sum of multiple circuits then one can
learn each individual circuit separately, leading to an exponential complexity in the size of the largest subcircuit, which can be much smaller. Early inquiries
suggest that each saddle escape roughly corresponds to the learning
of a new sub-circuit, or a step of the Frank-Wolfe algorithm.

Even more exciting, the subadditivity under composition could make
it possible to learn a composition of simple circuits by learning
each sub-circuits separately. Previous results indicate that this
type of compositional learning is possible, but it seems to require
the presence of correlations between the intermediate steps and the
outputs to guide the learning process \citep{malach_2024_autoregr_learn_algo,cagnetta_2024_hierarchy_correlations}.
This seems to match the way humans learn and teach: explaining intermediate
steps can make extremely complex ideas learnable.

if successful, this would describe DNNs as highly general
statistical models, capable of learning any circuit/algorithm as long
as it can be decomposed into the sum and composition of simple circuits,
thus describing a very large family of tasks all learnable with the
same model. This could explain why DNNs are the only models that have
been able to approach human intelligence in its generality.

\bibliographystyle{plain}
\bibliography{main}

\end{document}